\documentclass{article}

 \usepackage[preprint]{neurips_2026}

\usepackage[utf8]{inputenc} 
\usepackage[T1]{fontenc}    
\usepackage{hyperref}       
\usepackage{url}            
\usepackage{booktabs}       
\usepackage{amsfonts}       
\usepackage{nicefrac}       
\usepackage{microtype}      
\usepackage{xcolor}         

\usepackage{algorithmic}
\usepackage{algorithm}
\usepackage{mathtools}
\usepackage{wrapfig}
\usepackage{subcaption}
\usepackage{enumitem}

\usepackage{amsmath,amsfonts,bm}

\def\eqref#1{equation~\ref{#1}}

\def\1{\bm{1}}

\def\vtheta{{\bm{\theta}}}
\def\vphi{{\bm{\phi}}}
\def\va{{\bm{a}}}

\def\vg{{\bm{g}}}

\def\vJ{{\bm{J}}}

\def\vp{{\bm{p}}}
\def\vq{{\bm{q}}}

\def\vs{{\bm{s}}}

\def\vu{{\bm{u}}}
\def\vv{{\bm{v}}}

\DeclareMathAlphabet{\mathsfit}{\encodingdefault}{\sfdefault}{m}{sl}
\SetMathAlphabet{\mathsfit}{bold}{\encodingdefault}{\sfdefault}{bx}{n}

\def\sA{{\mathcal{A}}}

\def\sD{{\mathcal{D}}}

\def\sL{{\mathcal{L}}}

\def\sS{{\mathcal{S}}}
\def\sT{{\mathcal{T}}}

\newcommand{\R}{\mathbb{R}}

\DeclareMathOperator*{\argmax}{arg\,max}
\DeclareMathOperator*{\argmin}{arg\,min}

\newcommand{\KL}{\textsc{kl}}

\usepackage{soul}
\usepackage[dvipsnames]{xcolor}

\usepackage{amsmath}
\usepackage{amssymb}
\usepackage{mathtools}
\usepackage{amsthm}
\usepackage{adjustbox}

\newtheorem{theorem}{Theorem}[section]

\newtheorem{lemma}[theorem]{Lemma}

\newtheorem{definition}[theorem]{Definition}

\usepackage[toc,page,header]{appendix}
\usepackage{minitoc}

\title{Distributionally Robust Multi-Task Reinforcement Learning via Adaptive Task Sampling}

\author{%
  Nicholas E.~Corrado \\
 Computer Sciences Department \\
 University of Wisconsin -- Madison \\
 \texttt{ncorrado@wisc.edu}
  \And
 Wenyuan Huang \\
 Computer Sciences Department \\
 University of Wisconsin -- Madison \\
 \texttt{whuang369@wisc.edu}
  \And
 Josiah P.~Hanna \\
 Computer Sciences Department \\
 University of Wisconsin -- Madison \\
 \texttt{jphanna@cs.wisc.edu}
}

\begin{document}

\doparttoc 
\faketableofcontents 

\maketitle

\begin{abstract}
Multi-task reinforcement learning (MTRL) aims to train a single agent to efficiently optimize performance across multiple tasks simultaneously.
However, jointly optimizing all tasks often yields imbalanced learning: agents quickly solve easy tasks but learn slowly on harder ones.
While prior work primarily attributes this imbalance to conflicting task gradients and proposes gradient manipulation or specialized architectures to address it, we instead focus on a distinct and under-explored challenge: \emph{imbalanced data allocation}.
Standard MTRL allocates an equal number of environment interactions to each task, which over-allocates data to easy tasks that require relatively few interactions to solve and under-allocates data to hard tasks that require substantially more experience to solve.
To address this challenge, we introduce \textbf{D}istributionally \textbf{R}obust \textbf{A}daptive \textbf{T}ask \textbf{S}ampling (DRATS), an algorithm that adaptively prioritizes sampling tasks furthest from being solved.
We derive DRATS by formalizing MTRL as a feasibility problem from which we derive a minimax objective for minimizing the worst-case return gap, the difference between a desired target return and the agent's return on a task. 
In benchmarks like MetaWorld-MT10 and MT50, DRATS improves data efficiency and increases worst-task performance compared to existing task sampling algorithms.
\end{abstract}

\section{Introduction}

Multi-task reinforcement learning (RL) promises improved data efficiency by training a single 
agent to simultaneously solve many tasks.
In principle, training on many structurally similar tasks enables knowledge transfer, reducing number of interactions required to achieve strong performance~\citep{brunskill2013sample}.
In practice, however, jointly optimizing multiple tasks leads to imbalanced learning: agents 
quickly solve easy tasks but learn slowly on harder ones.
Prior work primarily attributes this imbalance to conflicting task gradients~\citep{yu2020gradient, 
liu2021conflict, navon2022multi, liu2023famo}, motivating methods that manipulate gradients, 
reweight the learning objective, or use specialized architectures to facilitate transfer while 
mitigating conflict~\citep{yang2020multi, sodhani2021multi, sun2022paco, hendawy2023multi}.
In this work, we instead focus on a distinct and under-explored cause of imbalanced multi-task learning: 
\emph{imbalanced data allocation}.

Standard multi-task training allocates equal interaction budgets to all tasks.
When tasks vary in difficulty, this allocation is inefficient: hard tasks require more 
experience to solve than easy ones~\citep{cho2024hard}.
%
Uniform task interaction therefore over-allocates data to easy tasks and under-allocates it to hard 
ones, ultimately causing imbalanced learning.
Curriculum learning (CL) methods~\citep{narvekar2020curriculum} adapt the task-sampling 
distribution over time, but typically prioritize easy tasks first assuming that positive transfer will facilitate learning on harder tasks.
In settings where tasks do not exhibit significant positive transfer, this strategy can exacerbate data imbalance by over-allocating data to easy tasks that need it least while neglecting harder tasks that need it most.

These observations suggest that a data-efficient multi-task learner should 
prioritize interacting with tasks \emph{furthest from being solved}.
Distributionally robust optimization (DRO) methods take a similar perspective, optimizing worst-task performance by up-weighting the learning objective for tasks with the largest loss~\citep{sagawa2019distributionally, oren2019distributionally}.
While DRO has been studied extensively in supervised learning~\citep{corrado2025automixalign, xie2024doremi, he2024robust, albalak2023efficient, xia2023sheared}, extensions to multi-task RL are relatively under-explored~\citep{mandal2025distributionally, panaganti2026group}.
Moreover, most existing DRO methods prioritize tasks by reweighting the learning objective while sampling tasks uniformly and thus still suffer from data imbalance.

In this work, we introduce a DRO-inspired multi-task learning algorithm that balances performance across all tasks in a data efficient manner. We first formalize multi-task RL as a feasibility problem in which the agent must reach a target return in each task and then derive a minimax optimization that minimizes the maximum \textit{return gap}, the difference between a desired target return on a task and the agent’s current return. Our algorithm, \textbf{D}istributionally \textbf{R}obust \textbf{A}daptive \textbf{T}ask \textbf{S}ampling (DRATS), adapts the task-sampling distribution to prioritize interacting with tasks in which the agent is furthest from its target returns. Because DRATS only modifies data collection, it can be combined with existing multi-task learning methods, such as those that use specialized architectures (\textit{e.g.},~\citep{hendawy2023multi}).
Empirically, DRATS yields more data efficient learning, higher asymptotic performance, and higher worst-task return on MetaWorld-MT10 and MT50 benchmarks~\citep{yu2020meta, mclean2025meta} as well as a multi-task version of the standard 6-task MuJoCo benchmark~\citep{todorov2012mujoco}.
In summary, our contributions are:
\begin{enumerate}
\item We introduce DRATS, a multi-task RL algorithm that adaptively prioritizes sampling tasks with the largest return gap. Using standard online learning techniques, 
    we prove that DRATS converges to within $\epsilon$ error of the optimal worst-task return 
    gap at a standard $O(1/\sqrt{T})$ rate.
    \item We empirically demonstrate that DRATS achieves greater data efficiency, higher aggregate 
return, and higher worst-task return than existing curriculum learning baselines.
    \item We empirically demonstrate that DRATS is compatible with multi-task network architectures and improves data efficiency whether tasks have positive transfer or no transfer at all.
\end{enumerate}

\section{Related Work}

\paragraph{Curriculum Learning.}
Curriculum learning (CL)~\citep{bengio2009curriculum, narvekar2020curriculum} improves data efficiency by controlling the order in which an agent samples training data or tasks. Most CL methods target \textit{single-task} settings, decomposing a single complex task into a sequence of simpler sub-tasks~\citep{andrychowicz2017hindsight, florensa2018automatic, ren2019exploration, pong2019skew, eysenbach2020c, zhang2021c, shah2021rapid, chane2021goal, zhang2020automatic, jiang2021prioritized, huang2022curriculum, cho2023outcome, liu2024single}. 
We instead consider \textit{multi-task} settings where tasks are not necessarily sub-problems of a single task and thus may not exhibit positive transfer or may even conflict with each other.
Existing multi-task CL methods often adapt the task distribution heuristically using return or success rate thresholds~\citep{wang2019paired, wang2020enhanced, cho2024hard} or by prioritizing tasks  with the most learning progress~\citep{matiisen2019teacher, mysore2019reward, colas2019curious, portelas2020teacher, kanitscheider2021multi}. 
Most prioritize easy tasks first, assuming positive transfer will accelerate learning on harder ones, which is not necessarily the most data-efficient approach when transfer is weak or tasks conflict.
Moreover, balanced task performance is not a direct objective of these methods: improved worst-task performance is a hoped-for byproduct of transfer, whereas in DRATS it is the optimization target.
The most closely related CL method is SMT~\citep{cho2024hard}, which prioritizes tasks with the smallest returns via manually-specified thresholds. DRATS has a similar goal but automatically adapts the task-sampling distribution via minimax optimization and has convergence guarantees that SMT lacks.


\paragraph{Distributionally Robust Optimization.} 
DRATS is a
distributionally robust optimization (DRO) method. 
In supervised learning, DRO methods minimize the worst-case loss across tasks and are widely used to determine the most effective way to combine data from different sources or tasks~\citep{oren2019distributionally, 
sagawa2019distributionally, michel2021balancing, albalak2023efficient, xia2023sheared, 
xie2024doremi, he2024robust}.
To our knowledge, only two works study DRO for RL---both in RLHF settings with 
LLMs~\citep{mandal2025distributionally, panaganti2026group}---making DRATS the first application 
of DRO to multi-task RL in continuous control settings.
Nearly all DRO methods in both supervised and RL settings reweight the training objective to up-weight high-loss tasks while continuing to sample tasks uniformly, which means data imbalance is still a source of data inefficiency.\footnote{We empirically validate this statement in Appendix~\ref{app:ablations_reweight}.}
In contrast, DRATS adapts the task sampling distribution.
Our work is therefore most closely related to \citet{corrado2025automixalign} and 
\citet{albalak2023efficient}, which also adapt the sampling distribution rather than the objective, 
but do so for LLM alignment and pretraining in supervised settings, respectively.
%

\paragraph{Gradient Manipulation.} 
A core challenge in multi-task learning is \textit{gradient conflict}: gradients from 
different tasks may point in opposing directions, so updating in the direction of the average 
gradient can impede learning on some tasks.
Several methods mitigate conflict by finding a new gradient direction that simultaneously 
improves all tasks~\citep{desideri2009multiple, yu2020gradient, liu2021conflict} or balances 
the rate of progress across tasks~\citep{liu2021towards, navon2022multi, liu2023famo}.
DRATS is orthogonal to these methods: it only changes how much data is used to compute 
task gradients, leaving the expected gradient direction and magnitude unchanged.
Gradient manipulation and DRATS therefore target distinct sources of 
imbalanced learning.

\paragraph{Network Architectures.} Multi-task network architectures enable task transfer while also mitigating conflict.
CARE \citep{sodhani2021multi} uses contextual attention to dynamically weight shared representations. PaCo \citep{sun2022paco} decomposes the model into shared and task-specific parameters.
Soft Modularization \citep{yang2020multi} and MOORE \citep{hendawy2023multi} use mixture-of-experts to encode task representations.
\citet{mclean2025multi} find that simply scaling the number of parameters in multi-task architectures significantly improves performance. 
DRATS only adapts the amount of data collect for each task and can be applied on top of multi-task network architectures.

\section{Preliminaries: Multi-Task Reinforcement Learning}
We formalize an RL environment as a finite horizon Markov decision process (MDP)~\citep{puterman2014markov} $(\sS, \sA, p, r, d^0, \gamma)$ with state space $\sS$, action space $\sA$, transition dynamics $p : \sS\times\sA\times\sS\to[0,1]$,  reward function $r : \sS\times\sA\to\R$, initial state distribution $d_0$, and discount factor $\gamma \in [0, 1)$.
In the multi-task RL setting, we work with a collection of $k$ environments (or \textit{tasks}) indexed by $i \in \sT = \{1,\dots,k\}$. 
For simplicity notation, we assume tasks share the same state and action spaces but can have different initial state distributions $d^0_i$, transition dynamics $p_i$, and reward functions $r_i$.  
We consider a parameterized family of stochastic, task-conditioned policies
$\pi_\vtheta(\va | \vs, i)$, 
denoting
the probability of selecting action $\va$ in state $\vs$ in task $i$. We condition policies on a task by appending a one-hot encoded task identifier to the state.
We denote the expected discounted return of $\pi_\vtheta$ in a specific task $i$ as $J_i(\vtheta) = \mathbb E_{\tau \sim \pi_\vtheta}\left[\sum_{t=0}^{H} \gamma^t r_i(\vs_t,\va_t) \right]$,
where $\tau$ is a trajectory sampled from $\pi_\vtheta$ and $H$ is a random variable denoting the episode horizon. 
The standard MTRL objective is to learn a single policy that maximizes the expected discounted return averaged across tasks, $J(\vtheta) = \frac{1}{k}\sum_{i=1}^k J_i(\vtheta)$.

\section{Distributionally Robust Adaptive Task Sampling (DRATS)}
\label{sec:feasibility}

Maximizing the expected return averaged across tasks can yield imbalanced task performance, since an agent can achieve a high aggregate return by excelling on easy tasks while performing poorly on others.
In this section, we instead formalize multi-task RL as a feasibility problem in which the agent must reach a 
target return in each task, and show that this perspective yields an adaptive 
task-sampling algorithm that prioritizes tasks based on the gap between the agent's current 
return and target return.

\subsection{Multi-Task RL as a Feasibility Problem}

Suppose we have a vector of \textit{reference returns} $\vJ^\text{ref} = (J_1^{\text{ref}}, \dots, J_k^{\text{ref}})$ representing a desirable return in each task. For now, we treat $\vJ^\text{ref}$ as known and focus on deriving a new optimization problem; we discuss how to determine $\vJ^\text{ref}$ later.
We formalize the multi-task objective as a feasibility problem:
\begin{equation}
\text{Find } \vtheta \quad \text{s.t.} \quad J_i(\vtheta) \ge J_i^\text{ref}, \quad \forall i \in [k].
\label{eq:feasibility_problem}
\end{equation}
We transform Eq.~\ref{eq:feasibility_problem} into a minimization over constraint violations by 
defining the \textit{return gap} $g_i(\vtheta) \coloneqq \max\left\{ J_i^{\text{ref}} - J_i(\vtheta), 0 \right\}$,
which is positive when the constraint is violated and zero once it is satisfied.\footnote{For 
any task whose constraint is violated, minimizing the return gap is equivalent to maximizing 
the return.}
A policy solves all tasks if and only if $g_i(\vtheta) = 0$ for all $i$.
When constraints cannot be simultaneously satisfied---due to limited model capacity or 
conflicting task objectives---we assume equal task importance and spread violations across 
tasks as evenly as possible.
This goal is naturally captured by an exact penalty formulation which minimizes the maximum violation:
\begin{equation}
\label{eq:feas_minimax}
\min_{\vtheta}~\max_{i\in[k]} g_i(\vtheta).
\end{equation}
We can rewrite the inner maximization as a maximization over distributions of tasks:
\begin{equation}
\label{eq:feas_minimax_sampling}
\min_{\vtheta}~\max_{\vq \in \Delta^k} \mathbb{E}_{i \sim \vq} \left[ g_i(\vtheta) \right].
\end{equation}
where $\Delta^k = \left\{ \vq \in \mathbb{R}^k \mid \sum_{i=1}^k q_i = 1, q_i \ge 0 \right\}$ denotes the $k$-dimensional probability simplex. 
Eq.~\ref{eq:feas_minimax_sampling} recasts multi-task RL as an active data collection problem: $\vq$ defines the probability of interacting with each task, and the inner maximization adapts $\vq$ to 
prioritize tasks with the largest return gaps.
Solving the inner maximization over the full simplex $\Delta^k$ leads to a solution $\vq^*$ that concentrates all probability on the task with the largest gap. This solution has three drawbacks:
\textbf{(1)~Non-smoothness.} Since the most-violated task shifts throughout optimization, 
Eq.~\ref{eq:feas_minimax_sampling} is generally non-smooth, a known source of instability in 
minimax optimization~\citep{Razaviyayn2020NonconvexMO}.
\textbf{(2)~Data inefficiency.} Optimizing only the most-violated task at each update is 
greedy coordinate descent: it reduces that task's return gap but provides no guarantee of 
progress on others, which is inefficient when tasks share structure.
\textbf{(3)~Return gap estimation.} Return gaps $g_i(\vtheta)$ must be estimated online. 
If a task's probability $q_i$ becomes too small, its gap estimate will be high variance.
In the next section, we address these drawbacks by restricting the space of possible task distributions.

\subsection{KL-Regularized Task Sampling}
\label{sec:kl_regularization}

To address these challenges, we restrict the task distribution $\vq$ to lie within a $\mathrm{KL}$ ball of radius $\varepsilon$ centered at a base distribution $\vp_0$ which we take to be the uniform distribution, yielding the following constrained optimization problem:
\begin{equation}
\label{eq:kl_constrained_optimization}
\min_\vtheta\max_{\vq \in \mathcal{Q}} \mathbb{E}_{i \sim \vq} [g_i(\vtheta)], 
\qquad \text{where} \qquad 
\mathcal{Q} = \left\{
\vq\in\Delta^k \mid \mathrm{KL}(\vq\|\vp_0)\le \varepsilon
\right\}.
\end{equation}
The $\mathrm{KL}$ constraint prevents $\vq$ from collapsing all probability onto a single task.
We derive a closed-form solution to the inner maximization of Eq.~\ref{eq:kl_constrained_optimization} following a procedure similar to~\citet{peters2010relative, abdolmaleki2018maximum, peng2019advantage}. We provide a brief sketch here and a complete derivation in Appendix~\ref{app:derivation_softmax}.
We first form the unconstrained Lagrangian relaxation of Eq.~\ref{eq:kl_constrained_optimization}:
\begin{equation}
\label{eq:kl_penalty_optimization}
\min_\vtheta\max_{\vq \in \Delta^k} 
\left\{
\mathbb{E}_{i \sim \vq} [g_i(\vtheta)] 
- \frac{1}{\eta}\mathrm{KL}(\vq\|\vp_0)
\right\},
\end{equation}
where $\eta > 0$ is the inverse Lagrange multiplier. By strong duality, Eq.~\ref{eq:kl_constrained_optimization} and Eq.~\ref{eq:kl_penalty_optimization} have the same solution for an appropriate choice of $\eta$.
Differentiating the objective in Eq.~\ref{eq:kl_penalty_optimization} with respect to $\vq$, setting to zero, and solving for the optimal $\vq^*$ yields:
\begin{equation}
\label{eq:softmax_q}
q_i^*
=\frac{\exp\left(\eta g_i(\vtheta)\right)}{\sum_{j=1}^k\exp\left(\eta g_j(\vtheta)\right)}, \qquad \text{or equivalently,} \qquad \vq^* = \text{softmax}(\eta\vg(\vtheta))
\end{equation}
Here, $\eta$ acts as an inverse temperature controlling the sharpness of $\vq^*$: as 
$\eta \to \infty$, $\vq^*$ concentrates all mass on the task with the largest return gap, 
recovering the unregularized solution of Eq.~\ref{eq:feas_minimax_sampling}; as $\eta \to 0$, 
$\vq^*$ approaches the uniform distribution.
KL regularization  smooths the 
optimization by shifting the task distribution gradually rather than abruptly and maintains nonzero probability on 
all tasks (for finite $\eta$), enabling an agent to learn many tasks simultaneously and reliably estimate return gaps.
We optimize Eq.~\ref{eq:kl_penalty_optimization} using Algorithm~\ref{alg:dro}, interleaving 
gradient updates to $\vq$ and $\vtheta$ similar to~\citet{sagawa2019distributionally, 
xie2023doremi, corrado2025automixalign}.
Specifically, we fill the agent's buffer by sampling tasks from $\vq$, 
compute Monte Carlo return estimates for each task, and then update $\vq$ via mirror ascent~\citep{shalev2012online} on $\sL(\vtheta, \vq) = \mathbb{E}_{i \sim \vq} [g_i(\vtheta)] 
- \frac{1}{\eta}\mathrm{KL}(\vq\|\vp_0)$:
\begin{equation}
q_{t+1,i} = \frac{q_{t,i}\exp\left(\alpha[\nabla_\vq \sL(\vq_t)]_i\right)}
{\sum_{j=1}^k q_{t,j}\exp\left(\alpha[\nabla_\vq \sL(\vq_t)]_j\right)}
\qquad\quad
[\nabla_\vq \sL(\vq_t)]_i =
g_i(\vtheta)
-
\frac{1}{\eta}
\left(
\log\frac{q_{t,i}}{p_{0,i}}+1
\right)
\label{eq:mirror_ascent_update}
\end{equation}
where $\alpha$ is the step size. For completeness, we formally describe the mirror ascent update in Appendix~\ref{app:mirror_descent}.
Since sampling probabilities can approach zero even with KL regularization if $\eta$ is 
large, we project $\vq$ onto the simplex constrained to $[\varepsilon, 1]^k$ via a KL 
projection after each update (Line~\ref{alg:projection}).
%
Since returns across tasks may differ in scale, we also normalize return gaps to be in $[0, 1]$ (Line~\ref{alg:normalization}).


\paragraph{Choosing Reference Returns.}
Assuming tasks do not fundamentally conflict, $\vJ^\text{ref}$ should ideally denote the maximum achievable return in each task.\footnote{When tasks conflict and their maxima cannot be simultaneously achieved, the user may specify lower reference returns for some tasks to encode preferences over performance tradeoffs.}
Maximum returns are rarely known in advance, so we estimate them online as the maximum observed return.
If the agent is stuck in a low-return plateau, we note that the observed maximum return may underestimate the true maximum. 
When tasks have a notion of success, we resolve this issue by initializing $J_i^\text{ref} = r_\text{max} H_\text{max}$, where $r_\text{max}$ is the maximum per-step reward and $H_\text{max}$ is the maximum episode length, and then updating the reference to the maximum observed return once the agent's success rate exceeds, \textit{e.g.}, 0.5, at which point the agent is solving the task and the observed maximum is reliable.
Requiring an estimate of the maximum return is no stronger an assumption than those made by closely related baselines, which require user-specified return or success rate thresholds~\citep{cho2024hard, kanitscheider2021multi, wang2020enhanced}; a key distinction is that we can estimate ours \emph{online}.

\begin{algorithm}[t]
\caption{Distributionally Robust Adaptive Task Sampling (DRATS)}
\label{alg:dro}
\begin{algorithmic}[1]
\STATE \textbf{Inputs:} DRATS learning rate $\alpha \in [0,1]$, minimum sampling probability $\varepsilon$
\STATE \textbf{Outputs:} Trained model parameters $\vtheta$.
\STATE Initialize model parameters $\vtheta$, task distribution $\vq_1 = (q_1,\dots,q_k) \gets (\frac{1}{k},\dots,\frac{1}{k})$, buffer $\sD\gets \emptyset$
\FOR{iteration $t = 1,\dots,T$} 
\WHILE{$\sD$ is not full}
    \STATE Sample task $i \sim \vq_t$, collect a trajectory from task $i$, and add it to $\sD$.
\ENDWHILE
    \STATE Compute a Monte Carlo estimate of the return $\widehat J_i$ for each task $i\in [k]$
    \STATE Update reference returns $\vJ^\text{ref}$ (if applicable)
    \STATE Compute the normalized return gap $\widehat g_i = (J^\text{ref}_i - \widehat J_i)/(J^\text{ref}_i - J^\text{rand}_i)$ for each task $i\in[k]$.\label{alg:normalization}
    \STATE Update $\vq_t$ via exponentiated gradient ascent on $\sL(\vtheta, \vq) = \mathbb{E}_{i \sim \vq} [g_i(\vtheta)] 
- \frac{1}{\eta}\mathrm{KL}(\vq\|\vp_0)$: 
\begin{equation*}
[\nabla_\vq \sL(\vq_t)]_i
=
g_i(\vtheta)
-
\frac{1}{\eta}
\left(
\log\frac{q_{t,i}}{p_{0,i}}+1
\right) \qquad
q_{t+1,i} = \frac{q_{t,i}\exp\left(\alpha[\nabla_\vq \sL(\vq_t)]_i\right)}
{\sum_{j=1}^k q_{t,j}\exp\left(\alpha[\nabla_\vq \sL(\vq_t)]_j\right)}
\end{equation*}
    \STATE Enforce minimum sampling probability: $\vq_{t+1} \gets \argmin_{\vq \in \Delta^k,\, \vq \geq \varepsilon} \mathrm{KL}(\vq \| \vq_{t+1})$ \label{alg:projection}
    \STATE Update $\vtheta$ with an on-policy algorithm that tries to maximize $\mathbb{E}_{i \sim \vq}[g_i(\vtheta)] 
- \frac{1}{\eta}\mathrm{KL}(\vq\|\vp_0)$. \\
    \ENDFOR
\end{algorithmic}
\end{algorithm}

\subsection{Convergence Analysis}
\label{sec:convergence}

We now show that DRATS converges to the minimax optimal worst-task return gap 
at a rate of $O(1/\sqrt{T})$ under standard assumptions. Our analysis uses standard online learning proof techniques~\citep{duchi2023lecture, shalev2012online} and is similar to the analysis of \citet{corrado2025automixalign, sagawa2019distributionally}. 
We model the optimization of 
Eq.~\ref{eq:kl_penalty_optimization} as a zero-sum game between two players, 
the $\vtheta$-player and the $\vq$-player. For each round $t$, the $\vq$-player chooses $\vq_t \in \Delta^k$, the $\vtheta$-player chooses $\vtheta_t \in \Theta$, and then the $\vq$-player observes the vector of task gaps $g(\vtheta_t) = (g_1(\vtheta_t), \ldots, g_k(\vtheta_t))$ and runs mirror ascent on $\sL$. For the $\vtheta$-player, we require 
the following regret guarantee.
\begin{definition}
\label{def:low_regret}The $\vtheta$-player is $C$-low-regret if there exists $C > 0$ 
such that for any sequence $\vq_1, \ldots, \vq_T \in \Delta^k$,
\begin{equation}
\sum\nolimits_{t=1}^T \langle \vq_t, g(\vtheta_t)\rangle
\leq
\min_{\vtheta \in \Theta} \sum\nolimits_{t=1}^T \langle \vq_t, g(\vtheta)\rangle
+ C\sqrt{T}.
\end{equation}
\end{definition}
This assumption is satisfied by any no-regret online learning algorithm such as online 
gradient descent \citep{shalev2012online}---which in the context of 
RL is equivalent to REINFORCE~\citep{williams1992simple}.\footnote{Most of our experiments use PPO 
\citep{schulman2017proximal} since it is more data efficient than REINFORCE.}
Under this assumption, DRATS achieves the following convergence guarantee:
\begin{theorem}[Convergence of DRATS]
\label{thm:convergence}
Suppose $g_i(\vtheta) \in [0, M]$ for all $i \in [k]$ and $\vtheta \in \Theta$, and that the $\vtheta$-player is $C$-low-regret (Definition~\ref{def:low_regret}). For any target approximation error $\epsilon > 0$ w.r.t. 
$\min_{\vtheta\in\Theta}\max_{i\in[k]} g_i(\vtheta)$ (Eq.~\ref{eq:feas_minimax}), set $\eta = \frac{2\log k}{\epsilon}$ and $\alpha = \frac{\sqrt{2\log k}}{G\sqrt{T}}$ where $G \coloneqq M + \frac{1 + \max(\eta M, \log k)}{\eta} $. Then for any $T \geq \frac{4(G\sqrt{2\log k} + C)^2}{\epsilon^2}$, DRATS satisfies:
\begin{equation}
\label{eq:main_theorem}
\max_{i \in [k]} \frac{1}{T}\sum\nolimits_{t=1}^T g_i(\vtheta_t)
\leq
\min_{\vtheta \in \Theta} \max_{i \in [k]} g_i(\vtheta) + \epsilon.
\end{equation}
\end{theorem}
\begin{proof}
    See Appendix~\ref{app:convergence}.
\end{proof}

\section{Experiments}

We design experiments to test two core hypotheses:
\textbf{(H1)} DRATS will achieve a higher return aggregated across all tasks than baselines, or, if all methods converge to similar returns, DRATS will reduce the number of interactions required to match the returns of baselines;
\textbf{(H2)} DRATS will achieve higher worst-task returns than baselines.


\paragraph{Environments.} We evaluate on four benchmarks of increasing complexity: (1) a 4-task Gridworld (Fig.~\ref{fig:gridworld_diagram}); (2) 
MuJoCo6, comprising Swimmer-v4, Hopper-v4, HalfCheetah-v4, Walker2d-v4, Ant-v4, and 
Humanoid-v4~\citep{todorov2012mujoco}; (3) MetaWorld-MT10; and (4) MetaWorld-MT50, suites of 
10 and 50 robotic manipulation tasks with a Sawyer arm~\citep{yu2020meta, mclean2025meta}.
%

\paragraph{Baselines.}
We use several task-sampling baselines listed below. 
Since their original names generally do not reflect their prioritization strategy, we rename each baseline based on how it prioritizes tasks.
\begin{enumerate}[leftmargin=1.5em]
    \item \textbf{\textsc{Uniform}.} Tasks are sampled from a fixed uniform distribution throughout training. 

    \item \textbf{\textsc{Learning Progress} (\textit{e.g.}, ALP-GMM,~\citet{portelas2020teacher}; RIAC, \citet{baranes2009r}).} Prioritizes tasks with the largest absolute change 
    in Monte Carlo return between consecutive updates (\textit{i.e.}, the slope of the learning curve): $z_i = |\widehat{J}_{t,i} - \widehat{J}_{t-1,i}|$.
    %
    %
    
    \item \textbf{\textsc{Learning Potential} (PLR,~\citet{jiang2021prioritized}).}
    Prioritizes tasks with the largest $\ell_1$ value loss under GAE:
    $z_i = \frac{1}{T}\sum_{t=0}^{T}\left|\sum_{k=t}^{T}(\gamma\lambda)^{k-t}\delta_k\right|$,
    where $\delta_k = r_k + \gamma\widehat{V}(s_{k+1}) - \widehat{V}(s_k)$ is the TD-error at 
    step $k$. 
    %

    \item \textbf{Hard First (SMT,~\citet{cho2024hard}).} Samples uniformly from the $K$ unsolved tasks with the smallest current returns.
    A task is ``solved'' if its return exceeds threshold $M_i$ and is then replaced by the unsolved task with the smallest return.
    A task is ``unsolvable'' if its return fails to exceed threshold $m_i$ after 
    a fixed number of steps and is then removed from the active set.
    After $B_1$ steps, the agent samples uniformly from all 
    ``unsolvable'' tasks.
    %
    \item \textbf{Easy First (\textit{e.g.}, \citet{joshi2025benchmarking})}. Concentrates sampling on the easiest task and switches to the next hardest upon reaching a success rate of $0.9$. 
    We only use this baseline in Gridworld; it requires us to rank tasks by difficulty, and it is unclear how to rank all tasks in other benchmarks.

\end{enumerate}

To ensure learning differences reflect differences in prioritization rather than implementation differences, we use the same mirror ascent update rule as DRATS (Eq.~\ref{eq:mirror_ascent_update}) to update the task distribution in \textsc{Learning Progress} and \textsc{Learning Potential}, using $z_i$ in place of $g_i$. 
To prevent catastrophic forgetting, we enforce a minimum sampling probability in \emph{all} baselines: $\varepsilon = 0.02$ in Gridworld and $\varepsilon = 1/(4k)$ in all other experiments.
\textsc{Hard First} (SMT) and \textsc{Learning Potential} (PLR) were originally proposed for off-policy RL, while we use on-policy RL. We describe minor adaptations of these methods to the on-policy setting in Appendix~\ref{app:baselines_modifications}.
To estimate reference returns for DRATS in MT10 and MT50, we initialize $J_i^\text{ref} = r_\text{max} H_\text{max} = 5000$ and 
then update $J_i^\text{ref}$ to the maximum observed return once the agent's success rate exceeds $0.5$.
In MuJoCo6 where tasks have no notion of success, we set $J_i^\text{ref}$ to the maximum observed return.
In Gridworld, we simply use $J_i^\text{ref} = 1$ for all tasks.
We train agents using REINFORCE~\citep{williams1992simple} in Gridworld and PPO~\citep{schulman2017proximal} in all other benchmarks.
We provide hyperparameter tuning details for all baselines in Appendix~\ref{app:hyperparameters}.
We use the same PPO/REINFORCE hyperparameters for all baselines within each benchmark (also listed in  Appendix~\ref{app:hyperparameters}). We build on top of the \texttt{metaworld-algorithms} codebase~\citep{mclean2025meta} for MT10/MT50 and CleanRL~\citep{huang2022cleanrl} for MuJoCo6.
%


\begin{figure}[t]
    \centering
    \begin{subfigure}{0.32\linewidth}
        \includegraphics[width=\linewidth]{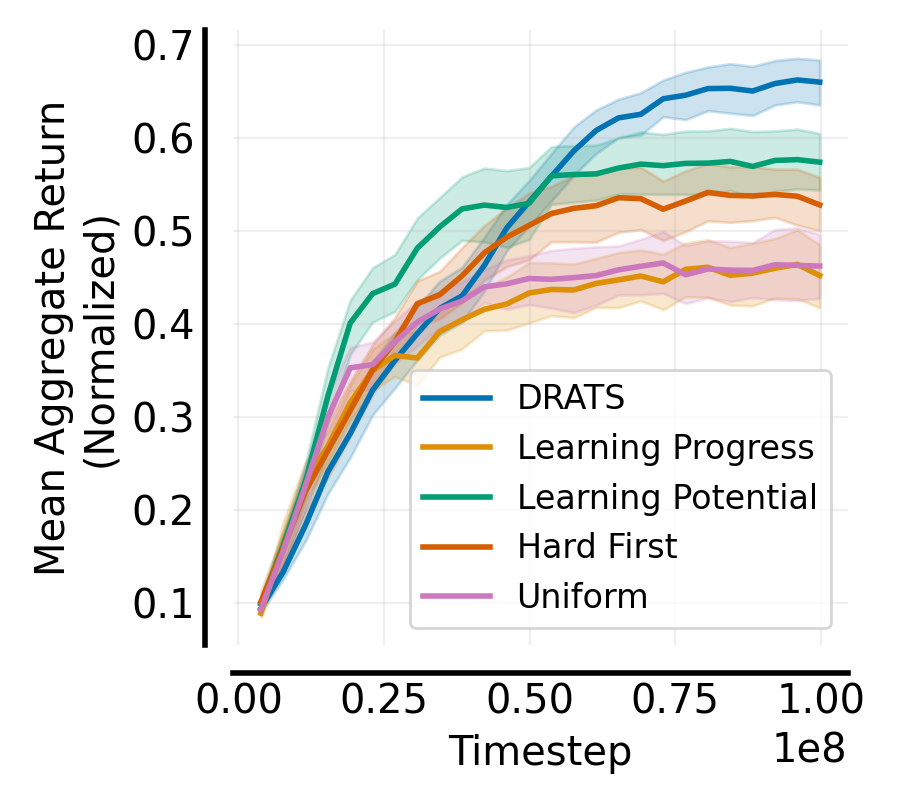}
    \caption{MuJoCo (20 seeds)}
    \label{fig:mujoco_return}
    \end{subfigure}
    \hfill
    \begin{subfigure}{0.32\linewidth}
        \includegraphics[width=\linewidth]{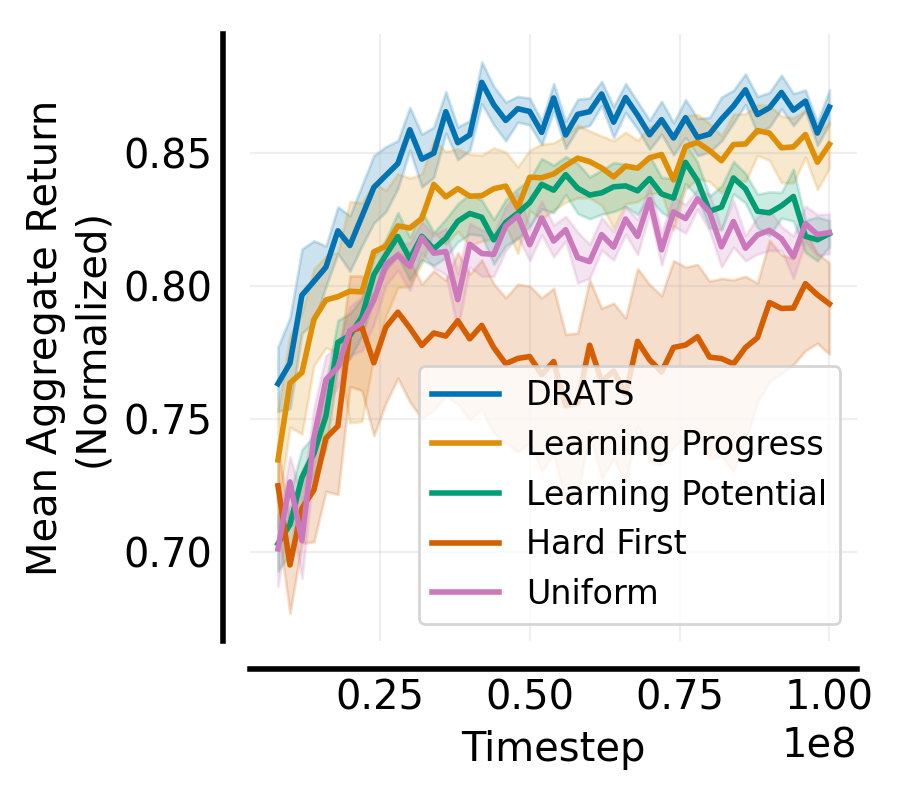}
    \caption{MT10 (20 seeds)}
    \label{fig:mt10_return}
    \end{subfigure}
    \hfill
    \begin{subfigure}{0.32\linewidth}
        \includegraphics[width=\linewidth]{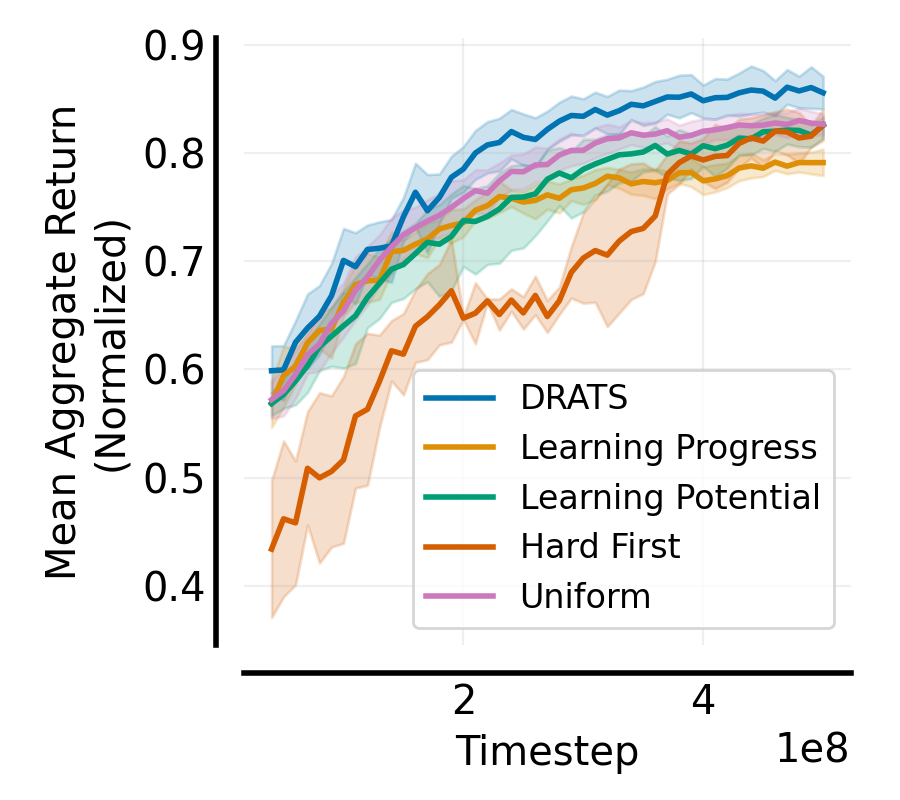}
    \caption{MT50 (10 seeds)}
    \label{fig:mt50_return}
    \end{subfigure}
    \caption{Mean normalized return aggregated over tasks in each benchmark. Shaded regions denote 95\% bootstrap confidence intervals.
    }
    \label{fig:return}
    \vspace{-0.5em}
\end{figure}

\begin{figure}[t]
    \centering
    \begin{subfigure}{0.24\linewidth}
        \includegraphics[width=\linewidth]{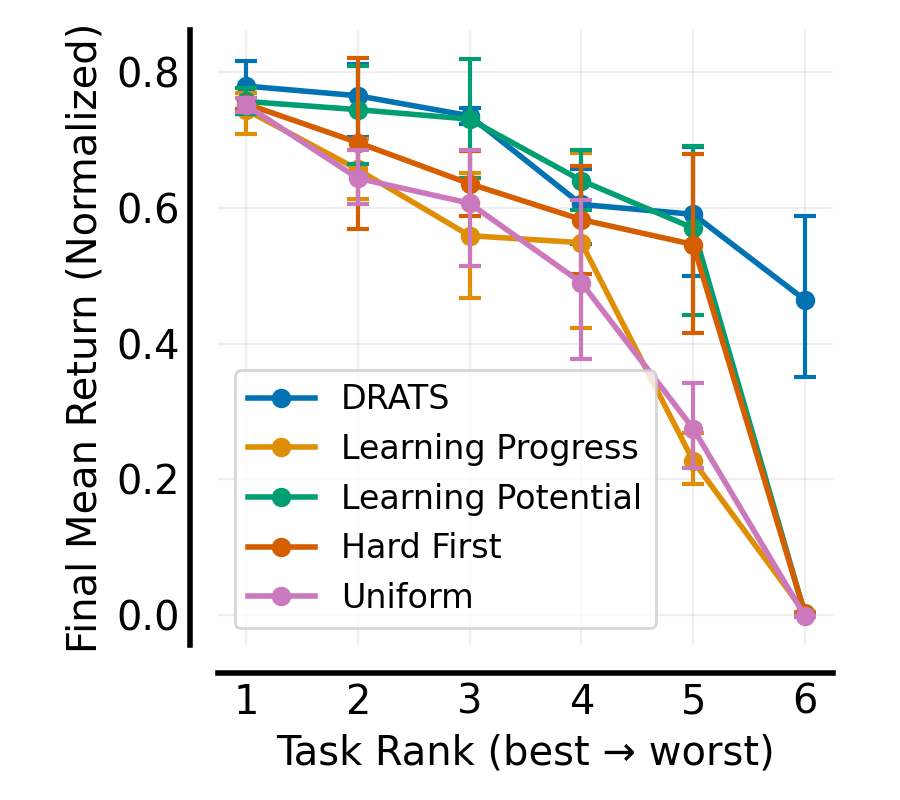}
    \caption{MuJoCo6 (20 seeds)}
    \label{fig:mujoco_task_final_return}
    \end{subfigure}
    \hfill
    \begin{subfigure}{0.24\linewidth}
        \includegraphics[width=\linewidth]{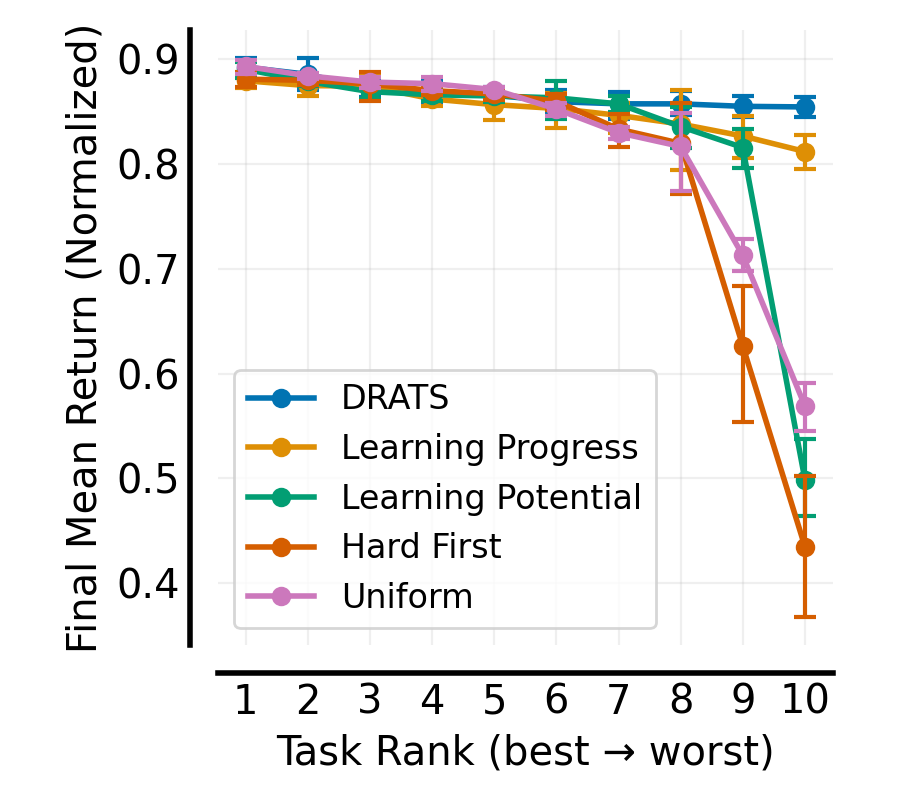}
    \caption{MT10 (20 seeds)}
    \label{fig:mt10_task_final_return}
    \end{subfigure}
    \hfill
    \begin{subfigure}{0.48\linewidth}
        \includegraphics[width=\linewidth]{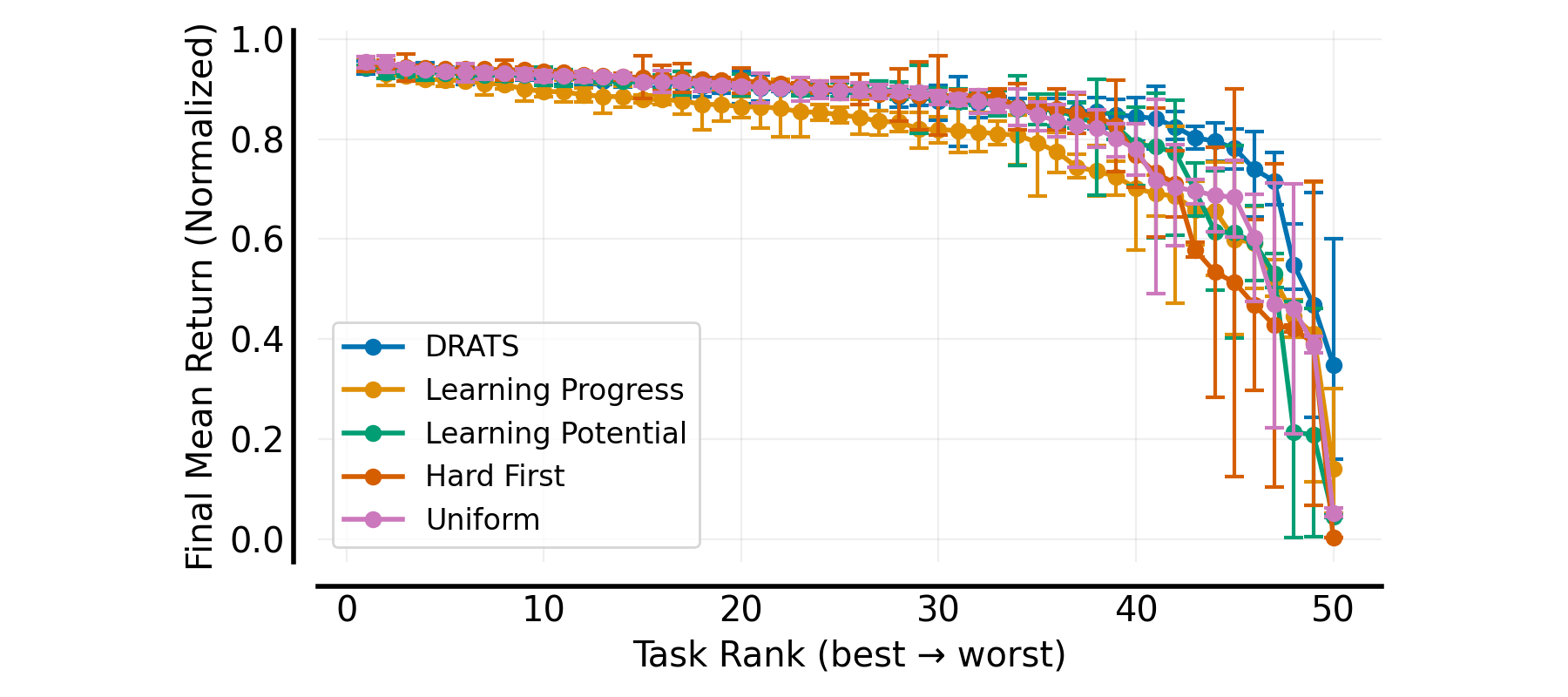}
    \caption{MT50 (10 seeds)}
    \label{fig:mt50_task_final_return}
    \end{subfigure}
    \caption{Final normalized return in each task, sorted from highest to lowest. 
    Error bars denote 95\% bootstrap confidence intervals. 
    }
    \label{fig:task_final_returns}
    \vspace{-1em}
\end{figure}

\paragraph{Per-Task Advantage Normalization.}
Standard advantage normalization computes mean and standard deviation over all data, aggregating across all tasks, which biases learning toward tasks with higher average return: if task~1 has a much larger mean return than task~2, global normalization makes task~2's advantages appear negative, which can make a policy gradient update decrease the probability of actions that are actually beneficial for task~2.
We instead normalize advantages for each task separately.
In Fig.~\ref{fig:adv_norm} of Appendix~\ref{app:experiments}, we show that per-task normalization improves data efficiency compared to global normalization.
We apply per-task normalization to \textit{all} methods. We do not treat it as a contribution of DRATS.


\paragraph{Metrics.}
Since maximum returns differ across tasks, we normalize task return as $\tilde{J}_i = (J_i - J_i^\text{min}) / (J_i^\text{max} - J_i^\text{min}) \in [0, 1]$.
\footnote{We prefer normalized returns over success rates since return can still improve after achieving a perfect success rate by solving tasks in fewer steps. We plot success rates for MT10/MT50 in Fig.~\ref{fig:mt_success_rate} in Appendix~\ref{app:experiments}.}
%
For \textbf{H1}, we report $\tilde{J}_i$ averaged across tasks; when methods reach similar final returns, we measure data efficiency as interactions required to converge.
For \textbf{H2}, we report the distribution of final per-task $\tilde{J}_i$.

\subsection{MuJoCo6, MT10, and MT50 Experiments}

Fig.~\ref{fig:return} shows that DRATS achieves the highest aggregate return in MuJoCo6, MT10, and MT50, supporting \textbf{H1}.
%
%
Fig.~\ref{fig:task_final_returns} shows that this improvement is driven by higher worst-task returns rather than higher returns on easier tasks, supporting \textbf{H2}.
%
%
We now discuss results in more detail.
Due to space constraints, we show returns and sampling probabilities for individual tasks in Appendix~\ref{app:experiments}.

\begin{wrapfigure}{R}{0.5\linewidth}    
    \centering
    \vspace{-1.5em}
    \includegraphics[width=\linewidth]{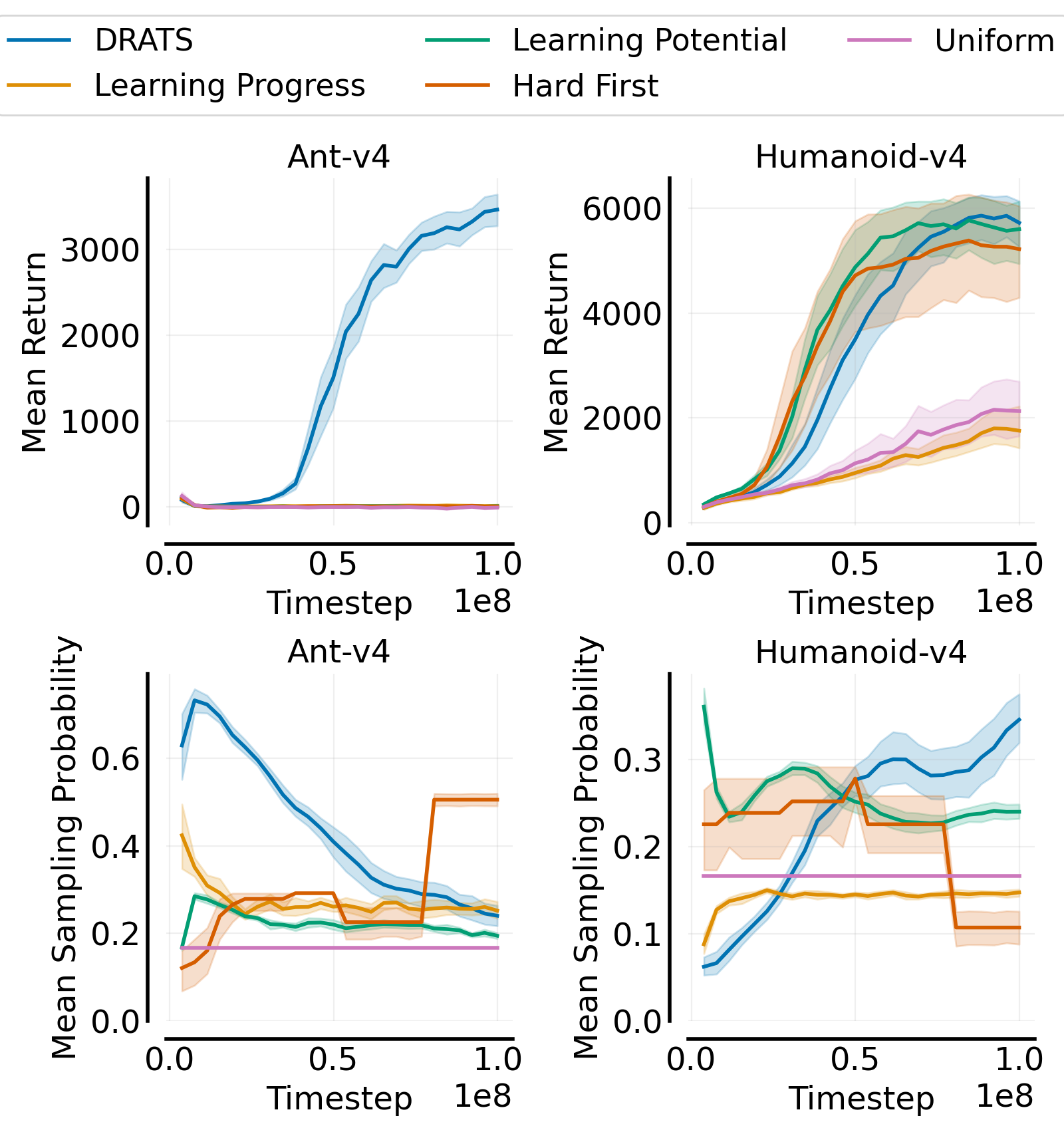}
    \caption{MuJoCo6 task returns and probabilities (20 seeds). Shaded regions denote 95\% bootstrap confidence intervals. }
    \label{fig:mujoco_hard_only}
        \vspace{-2em}
\end{wrapfigure}


\paragraph{MuJoCo6.}
Fig.~\ref{fig:mujoco_hard_only} shows results for the two hardest tasks, Ant and Humanoid; all methods perform similarly on the remaining four tasks (Fig.~\ref{fig:mujoco_task_results} in Appendix~\ref{app:experiments}), so performance differences are driven by these two.
DRATS is the only method to solve both.
DRATS solves Ant by allocating $40$--$70\%$ of sampling probability to it for the first $50$M timesteps. 
All baselines allocate approximately $30\%$ for most of training and fail to make any progress on it.
\textsc{Hard First} and \textsc{Learning Potential} solve Humanoid by allocating $25$--$35\%$ probability to it, while \textsc{Uniform} and \textsc{Learning Progress} allocate at most $17\%$ and thus learn more slowly.
\textsc{Learning Progress} learns more slowly than \textsc{Uniform} because the small slope of its learning curve causes it to \textit{deprioritize} Humanoid even though the task is far from being solved.

\paragraph{MT10.}
As shown in Fig.~\ref{fig:mt10_task_results} of Appendix~\ref{app:experiments}, performance differences are mostly driven by \texttt{pick-place}, \texttt{push}, and \texttt{peg-insert-side} tasks. DRATS shifts probability to them at the start of training and ultimately exceeds \textsc{Uniform}'s final return on them after just 6M timesteps.
\textsc{Learning Progress} and \textsc{Learning Potential} heavily prioritize easy tasks initially (\textit{e.g.}, window-open, window-close) and do not prioritize the harder tasks (\texttt{pick-place}, \texttt{push}) until much later in training.
%
%
\textsc{Hard First} performs poorly on \texttt{pick-place}, \texttt{push}, and \texttt{peg-insert-side}: once a task is declared ``solved'', it abruptly shifts probability away from it, causing performance to collapse and forcing the agent to relearn the task---visible as repeated peaks and dips in their learning curves (Fig.~\ref{fig:mt10_task_results}).

\paragraph{MT50.}
%
As shown in Fig.~\ref{fig:mt50_task_returns} and Fig.~\ref{fig:mt50_task_probs} of Appendix~\ref{app:experiments},
DRATS achieves significantly larger returns than all baselines on hard tasks like \texttt{disassemble} and  \texttt{pick-place-wall} because it places 2-3 times more probability on them.
%
\textsc{Learning Progress} and \textsc{Learning Potential} perform no better than \textsc{Uniform}: both prioritize tasks that are already making rapid progress or have high value estimation error, which tends to coincide with the easier tasks that DRATS deprioritizes. 
\textsc{Learning Progress} also achieves slightly lower final returns on many easy tasks because prioritization weakens near convergence as learning curves flatten. For instance, it reaches a return of approximately 3500 in \texttt{button-press-topdown} while all other baselines reach 3700. We observe similar behavior in tasks like \texttt{coffee-button, door-unlock, hand-insert}, etc. as well (Fig.~\ref{fig:mt50_task_returns} of Appendix~\ref{app:experiments}).
Similar to MT10, \textsc{Hard First} exhibits performance collapse on many hard tasks (\textit{e.g.}, \texttt{bin-picking, basketball, pick-place}) (Fig.~\ref{fig:mt50_task_returns} of Appendix~\ref{app:experiments}).

\begin{wrapfigure}{R}{0.3\linewidth}
\vspace{-4em}
\centering
    \includegraphics[width=\linewidth]{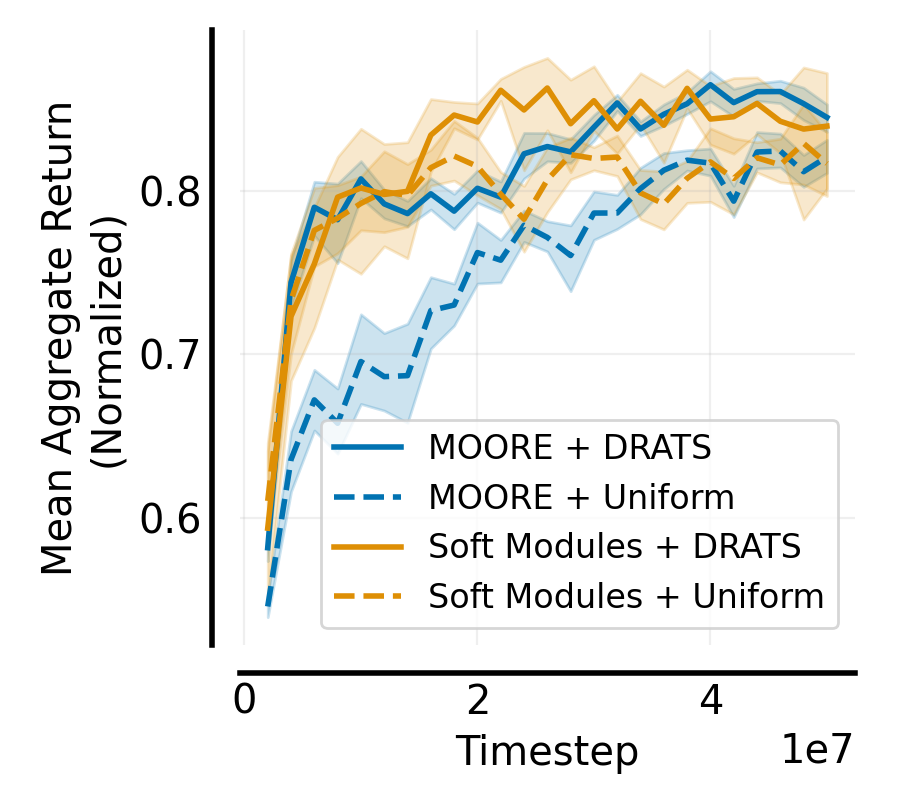}
    \caption{DRATS + MOORE and Soft Modularization in MT10 (10 seeds).}
    \label{fig:mt10_arch}
\vspace{0em}
\end{wrapfigure}

\paragraph{DRATS with Multi-Task Architectures.} 
Since DRATS only modifies data collection, it is orthogonal to multi-task network architectures such as Soft Modularization~\citep{yang2020multi} and MOORE~\citep{hendawy2023multi}. We hypothesize that applying DRATS on top of these architectures improves over using them with uniform task sampling. Fig.~\ref{fig:mt10_arch} verifies this hypothesis, showing that DRATS achieves a larger aggregate return than \textsc{Uniform} when using MOORE or Soft Modularization networks. We include an analogous experiment with MOORE on MT50 in Fig.~\ref{fig:mt50_moore} of Appendix~\ref{app:experiments} with qualitatively similar results.

\subsection{Gridworld Experiments}

\begin{figure}
    \centering
    \begin{subfigure}{0.28\linewidth}
    \includegraphics[width=\linewidth]{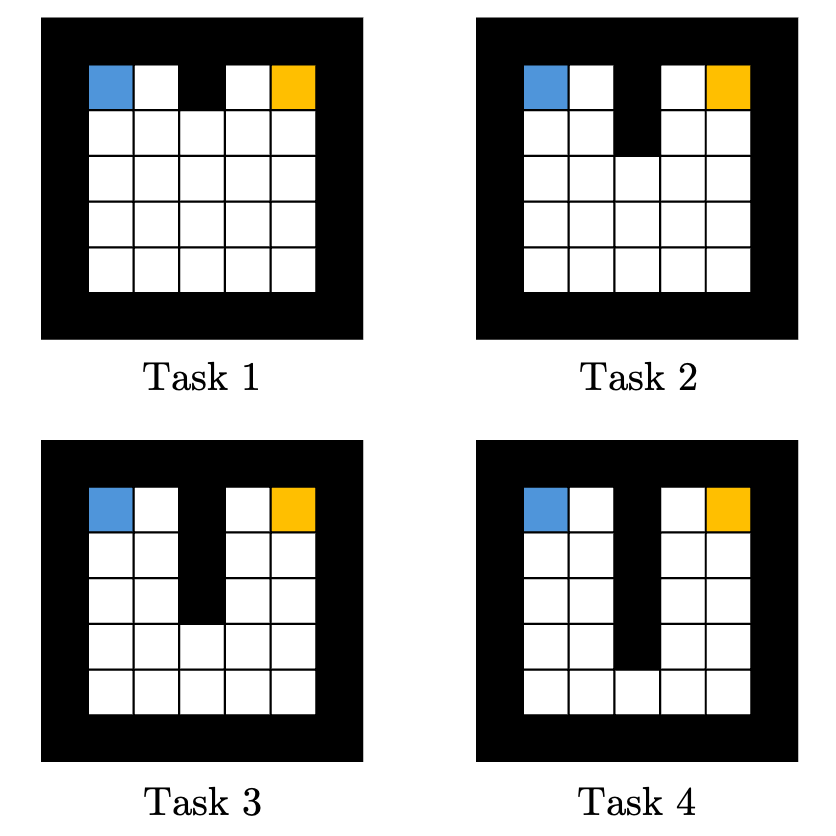}
    \caption{Gridworld tasks}
    \label{fig:gridworld_diagram}
    \end{subfigure} 
    \hfill
    \begin{subfigure}{0.32\linewidth}
    \includegraphics[width=\linewidth]{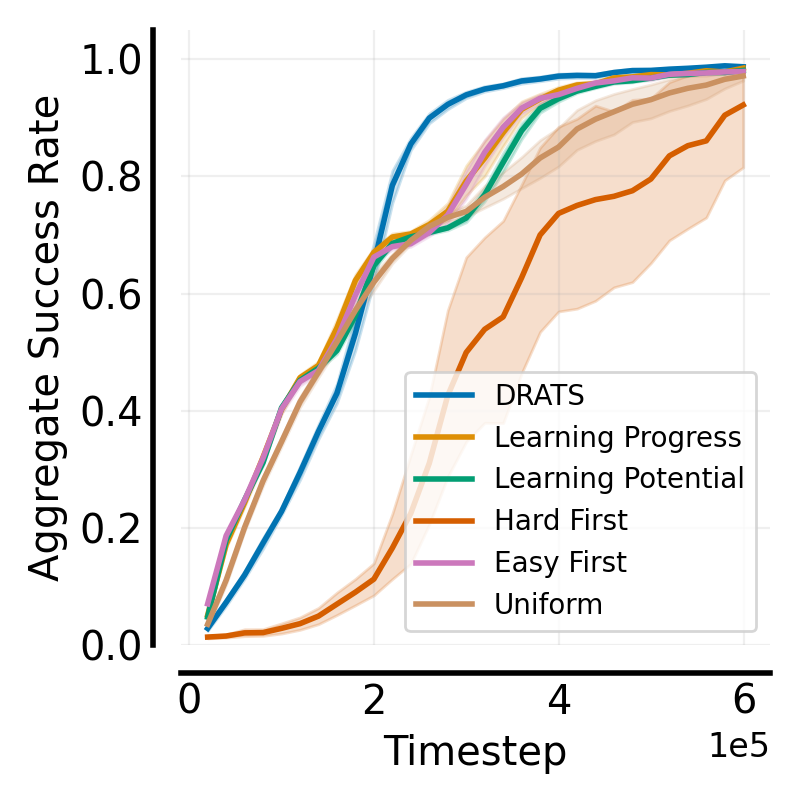}
    \caption{Gridworld, shared networks}
    \label{fig:gridworld_results}
    \end{subfigure}
    \hfill
    \begin{subfigure}{0.32\linewidth}
    \includegraphics[width=\linewidth]{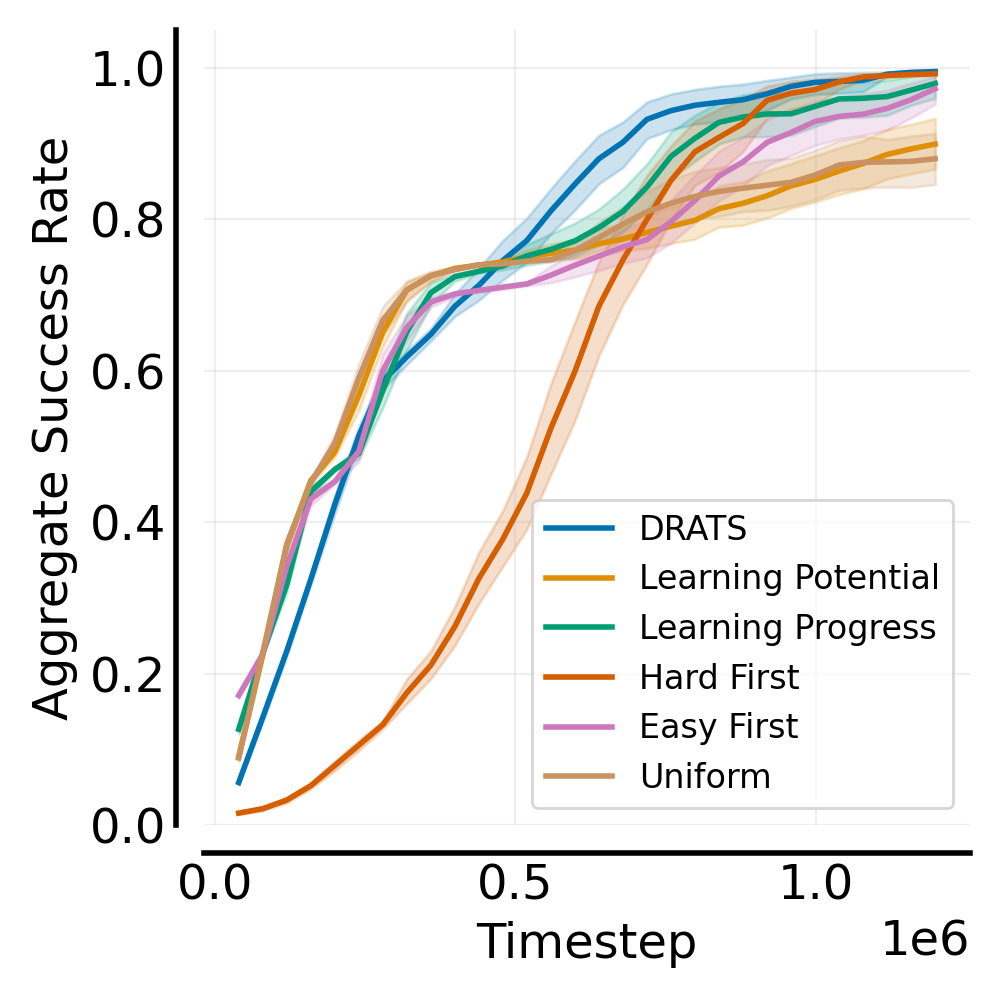}
    \caption{Gridworld, separate networks}
    \label{fig:gridworld_results_sep}
    \end{subfigure}
    \caption{
    \textbf{(a)} A 4-task Gridworld. The agent must navigate from the blue cell to the gold cell, receiving a reward of $+1$ for reaching the goal and reward $-0.001$ otherwise. We truncate episodes after 15 steps. The length of the shortest path to the goal increases from Task 1 to Task 4, so Task 4 is most difficult. 
    \textbf{(b)} Mean success rate over all tasks with shared actor/critic networks that enable positive transfer (50 seeds). \textbf{(c)} Mean success rate over all tasks with separate actor/critic networks for each task to eliminate transfer (50 seeds). Shaded regions denote 95\% bootstrap confidence intervals. 
    }
    \vspace{-0.5em}
\end{figure}

Curriculum learning (CL) methods prioritize easy tasks assuming that progress on easy tasks will accelerate learning on hard ones.
We now use a 4-task Gridworld (Fig.~\ref{fig:gridworld_diagram}) to illustrate how DRATS is more data efficient than CL methods when tasks exhibit positive transfer---and also when they do not.
Fig.~\ref{fig:gridworld_results_sep} shows results when training separate networks per task, which eliminates transfer entirely. Fig.~\ref{fig:gridworld_results} shows results using a shared network for all tasks, which enables positive transfer.\footnote{All methods solve all tasks in fewer steps with shared networks, confirming that transfer is positive.}
DRATS is the most data-efficient in both settings.
In the shared-network setting, DRATS prioritizes the hardest tasks while maintaining sufficient probability on easier tasks to benefit from transfer, solving all tasks simultaneously rather than sequentially (Fig.~\ref{fig:gridworld_tasks} and Fig.~\ref{fig:gridworld_probabilities} of Appendix~\ref{app:experiments}).
\textsc{Easy First}, \textsc{Learning Progress}, and \textsc{Learning Potential} do not prioritize Task~4 until Tasks~1--3 are mostly solved and thus converge more slowly.
\textsc{Hard First} converges slowly for the opposite reason: by neglecting easy tasks, it does not benefit as much from positive transfer that would help learn the hardest task.
DRATS outperforms all baselines even without transfer, showing that its advantage is not contingent on shared task structure.
These results also support \textbf{H1}.

\subsection{Ablations}

In Appendix~\ref{app:ablations}, we ablate the temperature $\eta$ and mirror ascent step size $\alpha$, compare per-task and global advantage normalization, and compare DRATS against a variant of DRATS that samples tasks uniformly but reweights each task's return by $q_i$, isolating the benefit of adaptive data collection from adaptive objective weighting.
We find that DRATS is relatively robust to the choice of $\eta$ and $\alpha$, that per-task normalization is more data efficient than global normalization, and that adaptive sampling is more data efficient than adaptive objective reweighting.

\section{Limitations}

DRATS has two main limitations.
First, DRATS requires reference returns, which we estimate online rather than requiring the user to specify hardcoded thresholds---a limitation of common curriculum learning methods~\citep{cho2024hard, kanitscheider2021multi, wang2020enhanced}.
Online estimation may underestimate reference returns early in training, causing DRATS to under-allocate data to tasks that can still improve.
Second, DRATS assumes a discrete task set, but tasks may also be continuous.
%
In principle, one could model the task distribution as a neural network $\vq_\phi$ with weights $\vphi$ that outputs the parameters of a distribution over tasks (\textit{e.g.}, a Gaussian) and then sample from $\vq_\phi$ via the reparameterization trick~\citep{kingma2013auto}.
We leave this extension to future work.

\section{Conclusion}
%
%

We study \textit{data imbalance} in multi-task reinforcement learning (MTRL). Standard MTRL allocates an equal interaction budget to all tasks, which over-allocates data to easy tasks and under-allocates data to hard ones, causing agents to learn easy tasks quickly and hard tasks slowly.
To enable balanced RL across tasks, we introduce Distributionally Robust Adaptive Task Sampling (DRATS), an algorithm that adaptively prioritizes sampling tasks furthest from being solved.
We derive DRATS by formalizing MTRL as a feasibility problem from which we derive a minimax objective for minimizing the worst-case return gap across tasks.
In benchmarks like MetaWorld-MT10 and MT50, DRATS improves data efficiency and increases worst-task performance compared to existing task sampling algorithms. 

\bibliography{main}

@inproceedings{corrado2025automixalign,
  title={AutoMixAlign: Adaptive Data Mixing for Multi-Task Preference Optimization in LLMs},
  author={Corrado, Nicholas E and Katz-Samuels, Julian and Devraj, Adithya M and Yun, Hyokun and Zhang, Chao and Xu, Yi and Pan, Yi and Yin, Bing and Chilimbi, Trishul},
  booktitle={Proceedings of the 63rd Annual Meeting of the Association for Computational Linguistics (Volume 1: Long Papers)},
  pages={20234--20258},
  year={2025}
}

@article{xie2023doremi,
  title={Doremi: Optimizing data mixtures speeds up language model pretraining},
  author={Xie, Sang Michael and Pham, Hieu and Dong, Xuanyi and Du, Nan and Liu, Hanxiao and Lu, Yifeng and Liang, Percy S and Le, Quoc V and Ma, Tengyu and Yu, Adams Wei},
  journal={Advances in Neural Information Processing Systems},
  volume={36},
  pages={69798--69818},
  year={2023}
}

@article{sagawa2019distributionally,
  title={Distributionally robust neural networks for group shifts: On the importance of regularization for worst-case generalization},
  author={Sagawa, Shiori and Koh, Pang Wei and Hashimoto, Tatsunori B and Liang, Percy},
  journal={arXiv preprint arXiv:1911.08731},
  year={2019}
}

@article{schulman2017proximal,
  title={Proximal policy optimization algorithms},
  author={Schulman, John and Wolski, Filip and Dhariwal, Prafulla and Radford, Alec and Klimov, Oleg},
  journal={arXiv preprint arXiv:1707.06347},
  year={2017}
}

@article{andrychowicz2017hindsight,
  title={Hindsight experience replay},
  author={Andrychowicz, Marcin and Wolski, Filip and Ray, Alex and Schneider, Jonas and Fong, Rachel and Welinder, Peter and McGrew, Bob and Tobin, Josh and Pieter Abbeel, OpenAI and Zaremba, Wojciech},
  journal={Advances in neural information processing systems},
  volume={30},
  year={2017}
}

@article{williams1992simple,
  title={Simple statistical gradient-following algorithms for connectionist reinforcement learning},
  author={Williams, Ronald J},
  journal={Reinforcement learning},
  pages={5--32},
  year={1992},
  publisher={Springer}
}

@book{puterman2014markov,
  title={Markov decision processes: discrete stochastic dynamic programming},
  author={Puterman, Martin L},
  year={2014},
  publisher={John Wiley \& Sons}
}

@article{huang2022cleanrl,
  author  = {Shengyi Huang and Rousslan Fernand Julien Dossa and Chang Ye and Jeff Braga and Dipam Chakraborty and Kinal Mehta and João G.M. Araújo},
  title   = {CleanRL: High-quality Single-file Implementations of Deep Reinforcement Learning Algorithms},
  journal = {Journal of Machine Learning Research},
  year    = {2022},
  volume  = {23},
  number  = {274},
  pages   = {1--18},
  url     = {http://jmlr.org/papers/v23/21-1342.html}
}

@inproceedings{nikishin2022primacy,
  title={The primacy bias in deep reinforcement learning},
  author={Nikishin, Evgenii and Schwarzer, Max and D’Oro, Pierluca and Bacon, Pierre-Luc and Courville, Aaron},
  booktitle={International Conference on Machine Learning},
  year={2022},
  organization={PMLR}
}

@inproceedings{todorov2012mujoco,
  title={Mujoco: A physics engine for model-based control},
  author={Todorov, Emanuel and Erez, Tom and Tassa, Yuval},
  booktitle={2012 IEEE/RSJ international conference on intelligent robots and systems},
  pages={5026--5033},
  year={2012},
  organization={IEEE}
}

@article{he2024robust,
  title={Robust multi-task learning with excess risks},
  author={He, Yifei and Zhou, Shiji and Zhang, Guojun and Yun, Hyokun and Xu, Yi and Zeng, Belinda and Chilimbi, Trishul and Zhao, Han},
  journal={arXiv preprint arXiv:2402.02009},
  year={2024}
}

@article{mclean2025multi,
  title={Multi-Task Reinforcement Learning Enables Parameter Scaling},
  author={McLean, Reginald and Chatzaroulas, Evangelos and Terry, Jordan and Woungang, Isaac and Farsad, Nariman and Castro, Pablo Samuel},
  journal={arXiv preprint arXiv:2503.05126},
  year={2025}
}

@inproceedings{oren2019distributionally,
  title={Distributionally Robust Language Modeling},
  author={Oren, Shay and others},
  year={2019},
  booktitle={Proceedings of the Conference on Empirical Methods in Natural Language Processing (EMNLP)},
}

@article{xie2024doremi,
  title={Doremi: Optimizing data mixtures speeds up language model pretraining},
  author={Xie, Sang Michael and Pham, Hieu and Dong, Xuanyi and Du, Nan and Liu, Hanxiao and Lu, Yifeng and Liang, Percy S and Le, Quoc V and Ma, Tengyu and Yu, Adams Wei},
  journal={Advances in Neural Information Processing Systems},
  volume={36},
  year={2024}
}

@article{michel2021balancing,
  title={Balancing average and worst-case accuracy in multitask learning},
  author={Michel, Paul and Ruder, Sebastian and Yogatama, Dani},
  journal={arXiv preprint arXiv:2110.05838},
  year={2021}
}

@article{desideri2009multiple,
  title={Multiple-Gradient Descent Algorithm for Multiobjective Optimization},
  author={Desideri, Jacopo},
  year={2009},
  journal={Journal of Optimization Theory and Applications},
  volume={142},
  number={3},
  pages={639--656},
  doi={10.1007/s10957-008-9504-1}
}

@article{navon2022multi,
  title={Multi-task learning as a bargaining game},
  author={Navon, Aviv and Shamsian, Aviv and Achituve, Idan and Maron, Haggai and Kawaguchi, Kenji and Chechik, Gal and Fetaya, Ethan},
  journal={arXiv preprint arXiv:2202.01017},
  year={2022}
}

@inproceedings{liu2021towards,
  title={Towards impartial multi-task learning},
  author={Liu, Liyang and Li, Yi and Kuang, Zhanghui and Xue, J and Chen, Yimin and Yang, Wenming and Liao, Qingmin and Zhang, Wayne},
  year={2021},
  organization={iclr}
}

@article{liu2021conflict,
  title={Conflict-averse gradient descent for multi-task learning},
  author={Liu, Bo and Liu, Xingchao and Jin, Xiaojie and Stone, Peter and Liu, Qiang},
  journal={Advances in Neural Information Processing Systems},
  volume={34},
  pages={18878--18890},
  year={2021}
}

@article{yu2020gradient,
  title={Gradient surgery for multi-task learning},
  author={Yu, Tianhe and Kumar, Saurabh and Gupta, Abhishek and Levine, Sergey and Hausman, Karol and Finn, Chelsea},
  journal={Advances in Neural Information Processing Systems},
  volume={33},
  pages={5824--5836},
  year={2020}
}

@article{an2020resampling,
  title={Why resampling outperforms reweighting for correcting sampling bias with stochastic gradients},
  author={An, Jing and Ying, Lexing and Zhu, Yuhua},
  journal={arXiv preprint arXiv:2009.13447},
  year={2020}
}

@article{shalev2012online,
  title={Online learning and online convex optimization},
  author={Shalev-Shwartz, Shai and others},
  journal={Foundations and Trends{\textregistered} in Machine Learning},
  volume={4},
  number={2},
  pages={107--194},
  year={2012},
  publisher={Now Publishers, Inc.}
}

@article{albalak2023efficient,
  title={Efficient online data mixing for language model pre-training},
  author={Albalak, Alon and Pan, Liangming and Raffel, Colin and Wang, William Yang},
  journal={arXiv preprint arXiv:2312.02406},
  year={2023}
}

@article{xia2023sheared,
  title={Sheared llama: Accelerating language model pre-training via structured pruning},
  author={Xia, Mengzhou and Gao, Tianyu and Zeng, Zhiyuan and Chen, Danqi},
  journal={arXiv preprint arXiv:2310.06694},
  year={2023}
}

@inproceedings{bengio2009curriculum,
  title={Curriculum learning},
  author={Bengio, Yoshua and Louradour, J{\'e}r{\^o}me and Collobert, Ronan and Weston, Jason},
  booktitle={Proceedings of the 26th annual international conference on machine learning},
  pages={41--48},
  year={2009}
}

@article{narvekar2020curriculum,
  title={Curriculum learning for reinforcement learning domains: A framework and survey},
  author={Narvekar, Sanmit and Peng, Bei and Leonetti, Matteo and Sinapov, Jivko and Taylor, Matthew E and Stone, Peter},
  journal={Journal of Machine Learning Research},
  volume={21},
  number={181},
  pages={1--50},
  year={2020}
}

@inproceedings{florensa2018automatic,
  title={Automatic goal generation for reinforcement learning agents},
  author={Florensa, Carlos and Held, David and Geng, Xinyang and Abbeel, Pieter},
  booktitle={International conference on machine learning},
  pages={1515--1528},
  year={2018},
  organization={PMLR}
}

@article{ren2019exploration,
  title={Exploration via hindsight goal generation},
  author={Ren, Zhizhou and Dong, Kefan and Zhou, Yuan and Liu, Qiang and Peng, Jian},
  journal={Advances in Neural Information Processing Systems},
  volume={32},
  year={2019}
}

@article{pong2019skew,
  title={Skew-fit: State-covering self-supervised reinforcement learning},
  author={Pong, Vitchyr H and Dalal, Murtaza and Lin, Steven and Nair, Ashvin and Bahl, Shikhar and Levine, Sergey},
  journal={arXiv preprint arXiv:1903.03698},
  year={2019}
}

@article{eysenbach2020c,
  title={C-learning: Learning to achieve goals via recursive classification},
  author={Eysenbach, Benjamin and Salakhutdinov, Ruslan and Levine, Sergey},
  journal={arXiv preprint arXiv:2011.08909},
  year={2020}
}

@article{zhang2021c,
  title={C-planning: An automatic curriculum for learning goal-reaching tasks},
  author={Zhang, Tianjun and Eysenbach, Benjamin and Salakhutdinov, Ruslan and Levine, Sergey and Gonzalez, Joseph E},
  journal={arXiv preprint arXiv:2110.12080},
  year={2021}
}

@article{shah2021rapid,
  title={Rapid exploration for open-world navigation with latent goal models},
  author={Shah, Dhruv and Eysenbach, Benjamin and Kahn, Gregory and Rhinehart, Nicholas and Levine, Sergey},
  journal={arXiv preprint arXiv:2104.05859},
  year={2021}
}

@inproceedings{chane2021goal,
  title={Goal-conditioned reinforcement learning with imagined subgoals},
  author={Chane-Sane, Elliot and Schmid, Cordelia and Laptev, Ivan},
  booktitle={International conference on machine learning},
  pages={1430--1440},
  year={2021},
  organization={PMLR}
}

@article{zhang2020automatic,
  title={Automatic curriculum learning through value disagreement},
  author={Zhang, Yunzhi and Abbeel, Pieter and Pinto, Lerrel},
  journal={Advances in Neural Information Processing Systems},
  volume={33},
  pages={7648--7659},
  year={2020}
}

@inproceedings{jiang2021prioritized,
  title={Prioritized level replay},
  author={Jiang, Minqi and Grefenstette, Edward and Rockt{\"a}schel, Tim},
  booktitle={International Conference on Machine Learning},
  pages={4940--4950},
  year={2021},
  organization={PMLR}
}

@article{huang2022curriculum,
  title={Curriculum reinforcement learning using optimal transport via gradual domain adaptation},
  author={Huang, Peide and Xu, Mengdi and Zhu, Jiacheng and Shi, Laixi and Fang, Fei and Zhao, Ding},
  journal={Advances in neural information processing systems},
  volume={35},
  pages={10656--10670},
  year={2022}
}

@article{cho2023outcome,
  title={Outcome-directed reinforcement learning by uncertainty \& temporal distance-aware curriculum goal generation},
  author={Cho, Daesol and Lee, Seungjae and Kim, H Jin},
  journal={arXiv preprint arXiv:2301.11741},
  year={2023}
}

@article{liu2024single,
  title={A single goal is all you need: Skills and exploration emerge from contrastive rl without rewards, demonstrations, or subgoals},
  author={Liu, Grace and Tang, Michael and Eysenbach, Benjamin},
  journal={arXiv preprint arXiv:2408.05804},
  year={2024}
}

@article{kanitscheider2021multi,
  title={Multi-task curriculum learning in a complex, visual, hard-exploration domain: Minecraft},
  author={Kanitscheider, Ingmar and Huizinga, Joost and Farhi, David and Guss, William Hebgen and Houghton, Brandon and Sampedro, Raul and Zhokhov, Peter and Baker, Bowen and Ecoffet, Adrien and Tang, Jie and others},
  journal={arXiv preprint arXiv:2106.14876},
  year={2021}
}

@article{wang2019paired,
  title={Paired open-ended trailblazer (poet): Endlessly generating increasingly complex and diverse learning environments and their solutions},
  author={Wang, Rui and Lehman, Joel and Clune, Jeff and Stanley, Kenneth O},
  journal={arXiv preprint arXiv:1901.01753},
  year={2019}
}

@inproceedings{wang2020enhanced,
  title={Enhanced poet: Open-ended reinforcement learning through unbounded invention of learning challenges and their solutions},
  author={Wang, Rui and Lehman, Joel and Rawal, Aditya and Zhi, Jiale and Li, Yulun and Clune, Jeffrey and Stanley, Kenneth},
  booktitle={International conference on machine learning},
  pages={9940--9951},
  year={2020},
  organization={PMLR}
}

@inproceedings{cho2024hard,
  title={Hard tasks first: Multi-task reinforcement learning through task scheduling},
  author={Cho, Myungsik and Park, Jongeui and Lee, Suyoung and Sung, Youngchul},
  booktitle={Forty-first International Conference on Machine Learning},
  year={2024}
}

@article{matiisen2019teacher,
  title={Teacher--student curriculum learning},
  author={Matiisen, Tambet and Oliver, Avital and Cohen, Taco and Schulman, John},
  journal={IEEE transactions on neural networks and learning systems},
  volume={31},
  number={9},
  pages={3732--3740},
  year={2019},
  publisher={IEEE}
}

@article{baranes2009r,
  title={R-iac: Robust intrinsically motivated exploration and active learning},
  author={Baranes, Adrien and Oudeyer, Pierre-Yves},
  journal={IEEE Transactions on Autonomous Mental Development},
  volume={1},
  number={3},
  pages={155--169},
  year={2009},
  publisher={IEEE}
}

@inproceedings{mysore2019reward,
  title={Reward-guided curriculum for robust reinforcement learning},
  author={Mysore, Siddharth and Platt, Robert and Saenko, Kate},
  booktitle={Workshop on Multi-task and Lifelong Reinforcement Learning at ICML},
  year={2019}
}

@inproceedings{colas2019curious,
  title={Curious: intrinsically motivated modular multi-goal reinforcement learning},
  author={Colas, C{\'e}dric and Fournier, Pierre and Chetouani, Mohamed and Sigaud, Olivier and Oudeyer, Pierre-Yves},
  booktitle={International conference on machine learning},
  pages={1331--1340},
  year={2019},
  organization={PMLR}
}

@inproceedings{portelas2020teacher,
  title={Teacher algorithms for curriculum learning of deep rl in continuously parameterized environments},
  author={Portelas, R{\'e}my and Colas, C{\'e}dric and Hofmann, Katja and Oudeyer, Pierre-Yves},
  booktitle={Conference on Robot Learning},
  pages={835--853},
  year={2020},
  organization={PMLR}
}

@article{mandal2025distributionally,
  title={Distributionally robust reinforcement learning with human feedback},
  author={Mandal, Debmalya and Sasnauskas, Paulius and Radanovic, Goran},
  journal={arXiv preprint arXiv:2503.00539},
  year={2025}
}

@article{panaganti2026group,
  title={Group Distributionally Robust Optimization-Driven Reinforcement Learning for LLM Reasoning},
  author={Panaganti, Kishan and Liang, Zhenwen and Yu, Wenhao and Mi, Haitao and Yu, Dong},
  journal={arXiv preprint arXiv:2601.19280},
  year={2026}
}

@article{brunskill2013sample,
  title={Sample complexity of multi-task reinforcement learning},
  author={Brunskill, Emma and Li, Lihong},
  journal={arXiv preprint arXiv:1309.6821},
  year={2013}
}

@inproceedings{sodhani2021multi,
  title={Multi-task reinforcement learning with context-based representations},
  author={Sodhani, Shagun and Zhang, Amy and Pineau, Joelle},
  booktitle={International conference on machine learning},
  pages={9767--9779},
  year={2021},
  organization={PMLR}
}

@article{sun2022paco,
  title={Paco: Parameter-compositional multi-task reinforcement learning},
  author={Sun, Lingfeng and Zhang, Haichao and Xu, Wei and Tomizuka, Masayoshi},
  journal={Advances in Neural Information Processing Systems},
  volume={35},
  pages={21495--21507},
  year={2022}
}

@article{yang2020multi,
  title={Multi-task reinforcement learning with soft modularization},
  author={Yang, Ruihan and Xu, Huazhe and Wu, Yi and Wang, Xiaolong},
  journal={Advances in Neural Information Processing Systems},
  volume={33},
  pages={4767--4777},
  year={2020}
}

@article{hendawy2023multi,
  title={Multi-task reinforcement learning with mixture of orthogonal experts},
  author={Hendawy, Ahmed and Peters, Jan and D'Eramo, Carlo},
  journal={arXiv preprint arXiv:2311.11385},
  year={2023}
}

@article{liu2023famo,
  title={Famo: Fast adaptive multitask optimization},
  author={Liu, Bo and Feng, Yihao and Stone, Peter and Liu, Qiang},
  journal={Advances in Neural Information Processing Systems},
  volume={36},
  pages={57226--57243},
  year={2023}
}

@inproceedings{yu2020meta,
  title={Meta-world: A benchmark and evaluation for multi-task and meta reinforcement learning},
  author={Yu, Tianhe and Quillen, Deirdre and He, Zhanpeng and Julian, Ryan and Hausman, Karol and Finn, Chelsea and Levine, Sergey},
  booktitle={Conference on robot learning},
  pages={1094--1100},
  year={2020},
  organization={PMLR}
}

@article{mclean2025meta,
  title={Meta-world+: An improved, standardized, rl benchmark},
  author={McLean, Reginald and Chatzaroulas, Evangelos and McCutcheon, Luc and R{\"o}der, Frank and Yu, Tianhe and He, Zhanpeng and Zentner, KR and Julian, Ryan and Terry, Jordan K and Woungang, Isaac and others},
  journal={arXiv preprint arXiv:2505.11289},
  year={2025}
}

@article{joshi2025benchmarking,
  title={Benchmarking massively parallelized multi-task reinforcement learning for robotics tasks},
  author={Joshi, Viraj and Xu, Zifan and Liu, Bo and Stone, Peter and Zhang, Amy},
  journal={arXiv preprint arXiv:2507.23172},
  year={2025}
}

@article{Razaviyayn2020NonconvexMO,
  title={Nonconvex Min-Max Optimization: Applications, Challenges, and Recent Theoretical Advances},
  author={Meisam Razaviyayn and Tianjian Huang and Songtao Lu and Maher Nouiehed and Maziar Sanjabi and Mingyi Hong},
  journal={IEEE Signal Processing Magazine},
  year={2020},
  volume={37},
  pages={55-66},
  url={https://api.semanticscholar.org/CorpusID:219687311}
}

@article{duchi2023lecture,
  title={Lecture notes on statistics and information theory},
  author={Duchi, John},
  journal={URL: https://https://web. stanford. edu/class/stats311/lecture-notes. pdf},
  year={2023}
}

@inproceedings{peters2010relative,
  title={Relative entropy policy search},
  author={Peters, Jan and Mulling, Katharina and Altun, Yasemin},
  booktitle={Proceedings of the AAAI Conference on Artificial Intelligence},
  volume={24},
  number={1},
  pages={1607--1612},
  year={2010}
}

@article{abdolmaleki2018maximum,
  title={Maximum a posteriori policy optimisation},
  author={Abdolmaleki, Abbas and Springenberg, Jost Tobias and Tassa, Yuval and Munos, Remi and Heess, Nicolas and Riedmiller, Martin},
  journal={arXiv preprint arXiv:1806.06920},
  year={2018}
}

@article{peng2019advantage,
  title={Advantage-weighted regression: Simple and scalable off-policy reinforcement learning},
  author={Peng, Xue Bin and Kumar, Aviral and Zhang, Grace and Levine, Sergey},
  journal={arXiv preprint arXiv:1910.00177},
  year={2019}
}

@article{kingma2013auto,
  title={Auto-encoding variational bayes},
  author={Kingma, Diederik P and Welling, Max},
  journal={arXiv preprint arXiv:1312.6114},
  year={2013}
}
\bibliographystyle{plainnat}

\newpage
\appendix
\onecolumn

\addcontentsline{toc}{section}{Appendix} 
\part{} 
\parttoc 



\newpage
\section{Derivation of the KL-Regularized Task Distribution}
\label{app:derivation_softmax}


In this section, we derive a closed-form solution to the inner maximization of Eq.~\ref{eq:kl_constrained_optimization} for a fixed $\vtheta$:
\begin{equation}
\label{eq:kl_constrained_objective_app}
\max_{\vq \in \mathcal{Q}} \mathbb{E}_{i \sim \vq} [g_i(\vtheta)] \qquad \text{where} \qquad \mathcal{Q} = \{ \vq \in \Delta^k \mid \KL(\vq || \vp_0) \le \epsilon \}.
\end{equation}
Our derivation follows that of \citet{peters2010relative, abdolmaleki2018maximum, peng2019advantage}.
First, we show that strong duality holds so that we can justifiably optimize the unconstrained Lagrangian dual. Next, we show that the solution to the inner maximization is a weighted softmax distribution.

\paragraph{Showing strong duality. }If $\epsilon=0$, the feasible set reduces to $\vq=\vp_0$ and the solution is trivially $\vp_0$. Henceforth, we assume $\varepsilon > 0$.
We begin by rewriting Eq.~\ref{eq:kl_constrained_objective_app} in primal form:
\begin{align}
\max_{\vq\in\mathbb{R}^k} \quad & \sum_{i=1}^k q_i g_i(\vtheta) \\
\text{s.t.} \quad & \sum_{i=1}^k q_i \log \frac{q_i}{p_{0,i}} \le \epsilon \\
& \sum_{i=1}^k q_i = 1.
\end{align}
The objective is linear (and hence continuous) in $\vq$, and the feasible set $\mathcal{Q}$ is convex and compact. Therefore, the primal problem admits at least one optimal solution. Moreover, since $\vq=\vp_0$ is strictly feasible (\textit{i.e.}, $\KL(\vp_0||\vp_0)=0<\epsilon$), Slater’s condition holds. Hence, strong duality applies and the dual optimum is attained. As a result, it suffices to solve the Lagrange dual and recover the maximizer of the Lagrangian at the optimal dual variables.

\paragraph{Closed-form solution to the inner maximization.} We construct the Lagrangian by introducing a multiplier $\mu \ge 0$ for the $\KL$ inequality constraint and a multiplier $\lambda \in \mathbb{R}$ for the equality constraint:
\begin{equation}
\label{eq:lagrangian_app}
\mathcal{L}(\vq,\mu,\lambda) = \sum_{i=1}^k q_i g_i(\vtheta)
- \mu\left(\sum_{i=1}^k q_i \log\frac{q_i}{p_{0,i}} - \epsilon\right)
+ \lambda\left(\sum_{i=1}^k q_i - 1\right).
\end{equation}
For fixed $\mu>0$ and $\lambda$, the Lagrangian is strictly concave in $\vq$ over the simplex due to the negative entropy term. Hence, the maximizer lies in the interior of the simplex and is characterized by first-order optimality conditions. Taking the derivative with respect to each $q_i$ and setting it to zero yields
\begin{equation}
\label{eq:stationarity_app}
0 = \frac{\partial \mathcal{L}}{\partial q_i} =
g_i(\vtheta) - \mu \left(\log\frac{q_i}{p_{0,i}} + 1\right) + \lambda.
\end{equation}
Solving Eq.~\ref{eq:stationarity_app} for $q_i$ directly gives
\begin{equation}
q_i = p_{0,i}\exp\left(\frac{g_i(\vtheta)+\lambda-\mu}{\mu}\right).
\end{equation}
Since the term $\exp\left(\frac{\lambda-\mu}{\mu}\right)$ is a constant, it can be absorbed into a normalization constant $C>0$:
\begin{equation}
q_i = C p_{0,i}\exp\left(\frac{g_i(\vtheta)}{\mu}\right).
\end{equation}
Define the inverse temperature parameter $\eta = 1/\mu \ge 0$. By enforcing the normalization constraint $\sum_i q_i = 1$, we eliminate the multiplier $\lambda$ and determine $C$, giving the closed-form solution
\begin{equation}
\label{eq:qstar_softmax_app}
q_i^*
=\frac{p_{0,i}\exp\left(\eta g_i(\vtheta)\right)}{\sum_{j=1}^k p_{0,j}\exp\left(\eta g_j(\vtheta)\right)}.
\end{equation}
For $\epsilon > 0$, the optimal $\eta$ is the unique positive value such that $\KL(\vq^* || \vp_0) = \epsilon$, unless $g_i(\boldsymbol{\theta})$ is constant across all $i$, in which case $\vq^* = \vp_0$.
By strong duality (which is guaranteed by Slater’s condition), letting $(\mu^*,\lambda^*)$ be dual-optimal implies that $\vq^*$ in Eq.~\ref{eq:qstar_softmax_app} is also a primal optimal solution to~ Eq.~\ref{eq:kl_constrained_objective_app}. For a uniform base distribution $\vp_0$, the solution reduces to:
\begin{equation}
\label{eq:qstar_softmax_app_unif}
q_i^*
=\frac{\exp\left(\eta g_i(\vtheta)\right)}{\sum_{j=1}^k\exp\left(\eta g_j(\vtheta)\right)}, \qquad \text{or equivalently,} \qquad \vq^* = \text{softmax}(\eta\vg(\vtheta))
\end{equation}

\section{Mirror Ascent Update}
\label{app:mirror_descent}
In this section, we formally describe the DRATS mirror ascent update to the task sampling distribution $\vq$. We additionally show this update can be written as a weighted geometric mean of $\vq_t$ and the best-response $\vq^*$ (Eq.~\ref{eq:qstar_softmax_app_unif}), a form we use in Appendix~\ref{app:convergence} to lower bound the task sampling probabilities $q_{t,i}$. Following \citet{duchi2023lecture}, the general mirror ascent update at round $t$ is
\begin{equation}
\label{eq:mirror_step}
\vq_{t+1} = \argmax_{\vq \in \Delta^k} \left\{ \langle h_t, \vq \rangle - \frac{1}{\alpha} D_\psi(\vq, \vq_t) \right\},
\end{equation}
where $h_t = \nabla_\vq \mathcal{L}(\vq_t)$ is the gradient of the ascent objective $\mathcal{L}(\vq) = \langle \vq, g(\vtheta_t)\rangle - \frac{1}{\eta}\KL(\vq\|\vp_0)$, $\psi$ is a mirror map, and $D_\psi$ is Bregman divergence induced by $\psi$. We use the negative entropy mirror map $\psi(\vq) = \sum_i q_i \log q_i$, which induces $D_\psi(\vq, \vq_t) = \KL(\vq\|\vq_t)$. Algorithm~\ref{alg:mirror_descent} describes the update we implement.

\begin{algorithm}[t]
\caption{DRATS Mirror Ascent Update}
\label{alg:mirror_descent}
\begin{algorithmic}[1]
\STATE \textbf{Input:} Step size $\alpha \in (0, \eta]$, inverse temperature $\eta > 0$.
\FOR{each round $t = 1, \ldots, T$}
    \STATE Compute the gradient:
    \begin{equation}
    \label{eq:q_grad_app}
    h_{t,i} = g_i(\vtheta_t) - \frac{1}{\eta}\left(\log(kq_{t,i}) + 1\right).
    \end{equation}
    \STATE Perform the update:
    \begin{equation}
    \label{eq:eg_generic_revised}
    q_{t+1,i} = \frac{q_{t,i}\exp\left(\alpha h_{t,i}\right)}{\sum_{j=1}^k q_{t,j}\exp\left(\alpha h_{t,j}\right)}.
    \end{equation}
\ENDFOR
\end{algorithmic}
\end{algorithm}

We now show that \eqref{eq:eg_generic_revised} can be written as a weighted geometric mean of $\vq_t$ and $\vq^*$. Substituting \eqref{eq:q_grad_app} into \eqref{eq:eg_generic_revised} and using $p_{0,i} = 1/k$,
\begin{align}
q_{t+1,i} &\propto q_{t,i} \exp\left(\alpha g_i(\vtheta) - \frac{\alpha}{\eta}\log(q_{t,i}) - \frac{\alpha}{\eta}\log (k) - \frac{\alpha}{\eta}\right) \\
&\propto \label{eq:drop_alpha_over_eta} q_{t,i} \exp\left(\alpha g_i(\vtheta) - \frac{\alpha}{\eta}\log(q_{t,i})\right) \\
&=q_{t,i}^{1-\alpha/\eta} \exp\left(\alpha g_i(\vtheta)\right) \\
&=q_{t,i}^{1-\alpha/\eta} \exp\left(\alpha \frac{\eta}{\eta}g_i(\vtheta)\right) \\
&=q_{t,i}^{1-\alpha/\eta} \exp\left(\eta g_i(\vtheta)\right)^{\frac{\alpha}{\eta}} \\
&\propto q_{t,i}^{1-\alpha/\eta}(q_i^*)^{\alpha/\eta}, \label{eq:polyak_beta_eta}
\end{align}
where in Eq.~\ref{eq:drop_alpha_over_eta} we dropped $\exp(\alpha/\eta)$ and $\exp(-\frac{\alpha}{\eta}\log k)$ since both are independent of $i$ and thus absorbed into normalization, and in Eq.~\ref{eq:polyak_beta_eta} we used $q_i^* \propto \exp(\eta g_i(\vtheta))$ (Eq.~\ref{eq:qstar_softmax_app_unif}).
The update is therefore a weighted geometric mean of $q_{t,i}$ and $q_i^*$ when $\eta \in (0, \eta]$, which implies $q_{t+1,i} \geq \min(q_{t,i}, q_i^*)$.

\newpage
\section{Convergence Analysis}
\label{app:convergence}

In this appendix, we prove Theorem~\ref{thm:convergence}. Our proof bounds the regret of each player in the two-player game separately and then combines them. For the $\vq$-player, we apply mirror \textit{descent} to the convex loss $\ell_t(\vq) = -\langle \vq, g(\vtheta_t)\rangle + \frac{1}{\eta}\KL(\vq\|\vp_0)$, which is equivalent to the mirror ascent update derived in Appendix~\ref{app:mirror_descent} since $\ell_t = -\mathcal{L}$. We begin by stating a standard mirror descent regret bound specialized to our setting. Then, we restate and prove our convergence theorem.

\begin{lemma}[Regret of Mirror Descent]
\label{lem:omd_regret}
(Special case of Theorem~18.2.5 in \citet{duchi2023lecture}.) Let $\ell_1, \ldots, \ell_T$ be an arbitrary sequence of convex functions with subgradients $h_t \in \partial \ell_t(\vq_t)$, and let $\vq_1, \ldots, \vq_T$ be generated by mirror descent with mirror map $\psi(\vq) = \sum_i q_i \log q_i$, which is $1$-strongly convex with respect to $\|\cdot\|_1$ and induces Bregman divergence $D_\psi(\vu, \vv) = \KL(\vu\|\vv) = \sum_i u_i \log \frac{u_i}{v_i}$. Then for any $\vu \in \Delta^k$,
\begin{equation}
\label{eq:omd_general_regret_lemma}
\sum_{t=1}^T \big(\ell_t(\vq_t)-\ell_t(\vu)\big) \leq \frac{D_\psi(\vu, \vq_1)}{\alpha} + \frac{\alpha}{2}\sum_{t=1}^T \|h_t\|_\infty^2,
\end{equation}
where $\|\cdot\|_\infty$ is the dual norm of $\|\cdot\|_1$.
\end{lemma}

\begin{theorem}[Convergence of DRATS]
\label{thm:convergence_app}
Suppose $g_i(\vtheta) \in [0, M]$ for all $i \in [k]$ and $\vtheta \in \Theta$, and that the $\vtheta$-player is $C$-low-regret (Definition~\ref{def:low_regret}). For any target accuracy $\epsilon > 0$,  set $\eta = \frac{2\log k}{\epsilon}$ and $\alpha = \frac{\sqrt{2\log k}}{G\sqrt{T}}$ where $G \coloneqq M + \frac{1 + \max(\eta M, \log k)}{\eta} $. Then for any $T \geq \frac{4(G\sqrt{2\log k} + C)^2}{\epsilon^2}$, DRATS satisfies:
\begin{equation}
\label{eq:main_theorem_app}
\max_{i \in [k]} \frac{1}{T}\sum_{t=1}^T g_i(\vtheta_t)
\leq
\min_{\vtheta \in \Theta} \max_{i \in [k]} g_i(\vtheta) + \epsilon.
\end{equation}
\end{theorem}

\begin{proof}
\textbf{Step 1: the $\vq$-player.}
At round $t$, the $\vq$-player minimizes the KL-regularized loss
\begin{equation}
\label{eq:q_loss_t}
\ell_t(\vq) = -\langle \vq, g(\vtheta_t)\rangle + \frac{1}{\eta}\KL(\vq\|\vp_0).
\end{equation}
via mirror descent with the negative entropy mirror map $\psi(\vq) = \sum_i q_i \log q_i$ with Bregman divergence $D_\psi(\vu, \vv) = \KL(\vu||\vv)$, step size $\alpha \in (0, \eta]$, and uniform initialization $\vq_1 = \vp_0$. Since $\ell_t$ is the sum of the linear term $\langle \vq, g(\vtheta_t)\rangle$ and $\frac{1}{\eta}\KL(\vq\|\vp_0)$, which is $1$-strongly convex with respect to $\|\cdot\|_1$ \citep{duchi2023lecture}, $\ell_t$ is convex. Its gradient at interior points of the simplex is
\begin{equation}
\label{eq:q_grad}
h_{t,i} = \frac{\partial \ell_t}{\partial q_i}(\vq_t) = -g_i(\vtheta_t) + \frac{1}{\eta}\big(\log\frac{q_{t,i}}{p_{0,i}}+1\big) = -g_i(\vtheta_t) + \frac{1}{\eta}\big(\log(k q_{t,i})+1\big),
\end{equation}
where we substituted $p_{0,i} = 1/k$. We now upper bound $\|h_t\|_\infty$ by deriving a lower bound on $\log q_{t,i}$.\footnote{Algorithm~\ref{alg:dro} enforces a minimum sampling probability $q_{t,i} \geq q_\text{min}$. However, this step is not required for convergence, so we instead derive the bound directly from the structure of the mirror descent update.} 
Since $\alpha \in (0, \eta]$, the mirror descent update takes the form $q_{t+1,i} \propto q_{t,i}^{1-\alpha/\eta}(q_i^*)^{\alpha/\eta}$ (Eq.~\ref{eq:polyak_beta_eta}), a weighted geometric mean of $q_{t,i}$ and $q_i^*$. 
Therefore, $q_{t+1,i} \geq \min(q_{t,i}, q_i^*) \geq \frac{e^{-\eta M}}{k}$ for all $t, i$, where the last inequality uses $g_i(\vtheta) \in [0, M]$. It follows that $\log(kq_{t,i}) \in [-\eta M, \log k]$, so
\begin{equation}
|h_{t,i}| \leq M + \frac{1 + \max(\eta M, \log k)}{\eta} =: G \qquad \implies \qquad \|h_t\|_\infty \leq G \quad \forall\, t
\end{equation}
\textit{i.e.}, $\ell_t$ is $G$-Lipschitz with respect to $\|\cdot\|_1$. Thus, we can  apply Lemma~\ref{lem:omd_regret}: For any $\vu \in \Delta^k$:
\begin{equation}
\label{eq:omd_general_regret}
\begin{split}
    \sum_{t=1}^T \big(\ell_t(\vq_t)-\ell_t(\vu)\big) 
    &\leq \frac{\KL(\vu\|\vq_1)}{\alpha} + \frac{\alpha}{2}\sum_{t=1}^T \|h_t\|_\infty^2 \\
    &\leq \frac{\log k}{\alpha} + \frac{\alpha G^2 T}{2},
\end{split}
\end{equation}
where we substituted $\KL(\vu\|\vq_1) \leq \log k$ (since $\vq_1$ is uniform) and $\|h_t\|_\infty^2 \leq G^2$. 
Setting $\vu = e_i$ and expanding $\ell_t(e_i) = -g_i(\vtheta_t) + \frac{\log k}{\eta}$, rearranging \eqref{eq:omd_general_regret} gives
\begin{equation}
\sum_{t=1}^T g_i(\vtheta_t) \leq \sum_{t=1}^T -\ell_t(\vq_t) + \frac{T\log k}{\eta} + \frac{\log k}{\alpha} + \frac{\alpha G^2 T}{2}.
\end{equation}
Substituting $-\ell_t(\vq_t) \leq \langle \vq_t, g(\vtheta_t)\rangle$ and dividing by $T$:
\begin{equation}
\frac{1}{T}\sum_{t=1}^T g_i(\vtheta_t) \leq \frac{1}{T}\sum_{t=1}^T \langle \vq_t, g(\vtheta_t)\rangle + \frac{\log k}{\eta} + \frac{\log k}{\alpha T} + \frac{\alpha G^2}{2}.
\end{equation}
Since this bound holds for all $i \in [k]$ and the right-hand side is independent of $i$, the inequality is preserved when taking $\max_{i\in[k]}$ on the left:
\begin{equation}
\label{eq:step1_final}
\max_{i \in [k]} \frac{1}{T}\sum_{t=1}^T g_i(\vtheta_t) \leq \frac{1}{T}\sum_{t=1}^T \langle \vq_t, g(\vtheta_t)\rangle + \frac{\log k}{\eta} + \frac{\log k}{\alpha T} + \frac{\alpha G^2}{2}.
\end{equation}

\textbf{Step 2: the $\vtheta$-player.}
Since $\frac{1}{\eta}\KL(\vq_t\|\vp_0)$ does not depend on $\vtheta$, the $\vtheta$-player's effective loss at round $t$ is $\langle \vq_t, g(\vtheta)\rangle$. Thus, we can apply our $C$-low-regret assumption (Definition~\ref{def:low_regret}). After dividing by $T$, we have:
\begin{equation}
\label{eq:theta_regret}
\frac{1}{T}\sum_{t=1}^T \langle \vq_t, g(\vtheta_t)\rangle \leq \min_{\vtheta \in \Theta} \frac{1}{T}\sum_{t=1}^T \langle \vq_t, g(\vtheta)\rangle + \frac{C}{\sqrt{T}}.
\end{equation}
Since $\vq_t \in \Delta^k$, $\langle \vq_t, g(\vtheta)\rangle$ is a convex combination of $\{g_i(\vtheta)\}_{i=1}^k$, we have $\langle \vq_t, g(\vtheta)\rangle \leq \max_{i \in [k]} g_i(\vtheta)$. Averaging over $t=1,\dots,T$ and taking $\min_{\vtheta \in \Theta}$ on both sides yields:
\begin{equation}
\label{eq:theta_comparator}
\min_{\vtheta \in \Theta} \frac{1}{T}\sum_{t=1}^T \langle \vq_t, g(\vtheta)\rangle \leq \min_{\vtheta \in \Theta} \max_{i \in [k]} g_i(\vtheta).
\end{equation}
Substituting \eqref{eq:theta_comparator} into \eqref{eq:theta_regret}:
\begin{equation}
\label{eq:step2_final}
\frac{1}{T}\sum_{t=1}^T \langle \vq_t, g(\vtheta_t)\rangle \leq \min_{\vtheta \in \Theta} \max_{i \in [k]} g_i(\vtheta) + \frac{C}{\sqrt{T}}.
\end{equation}

\textbf{Step 3: Combining the bounds.}
Substituting \eqref{eq:step2_final} into \eqref{eq:step1_final}:
\begin{equation}
\label{eq:combined}
\max_{i \in [k]} \frac{1}{T}\sum_{t=1}^T g_i(\vtheta_t) \leq \min_{\vtheta \in \Theta} \max_{i \in [k]} g_i(\vtheta) + \frac{\log k}{\eta} + \frac{\log k}{\alpha T} + \frac{\alpha G^2}{2} + \frac{C}{\sqrt{T}}.
\end{equation}
To balance the two $\alpha$-dependent terms, we choose $\alpha^* = \frac{\sqrt{2 \log k}}{G\sqrt{T}}$ so that $\frac{\log k}{\alpha T} = \frac{\alpha G^2}{2}$. Substituting into \eqref{eq:combined}:
\begin{equation}
\label{eq:after_alpha}
\max_{i \in [k]} \frac{1}{T}\sum_{t=1}^T g_i(\vtheta_t) \leq \min_{\vtheta \in \Theta} \max_{i \in [k]} g_i(\vtheta) + \frac{\log k}{\eta} + \frac{G\sqrt{2\log k} + C}{\sqrt{T}}.
\end{equation}
Setting $\eta = \frac{2\log k}{\epsilon}$ reduces the bias term to $\frac{\epsilon}{2}$. The $O(1/\sqrt{T})$ term reaches $\frac{\epsilon}{2}$ for $T \geq \frac{4(G\sqrt{2\log k} + C)^2}{\epsilon^2}$. Thus, for any $T$ satisfying this condition:
\begin{equation}
\label{eq:final_bound}
\max_{i \in [k]} \frac{1}{T}\sum_{t=1}^T g_i(\vtheta_t) \leq \min_{\vtheta \in \Theta} \max_{i \in [k]} g_i(\vtheta) + \epsilon. \qed
\end{equation}
\end{proof}

\newpage
\section{Baselines}

\label{app:baselines_modifications}

In this section, we describe minor modifications we make to the original implementations of baselines.

\paragraph{\textsc{Hard First} (SMT).} We make three modifications to SMT.
\begin{enumerate}
    \item \textbf{No Resets.} SMT~\citep{cho2024hard} periodically resets network parameters (without resetting the replay buffer) to mitigate primacy bias~\citep{nikishin2022primacy}. We do not implement network resets for two reasons.
First, resets are designed for off-policy learning: after a reset, the agent can quickly relearn from its replay buffer using an off-policy algorithm. In on-policy learning, there is no replay buffer---on-policy algorithms discard data between updates because off-policy historic data biases gradient updates---so resetting network parameters is equivalent to restarting training entirely.
Second, network resets are orthogonal to task sampling and could in principle be applied on top of any method (in an off-policy settings); including them for only SMT would confound the task-sampling comparison we aim to study.
\item \textbf{Task-Specific Thresholds.} We use task-specific thresholds $M_i$ since tasks achieve different returns when solved, and using a single shared threshold might cause SMT to declare a task solved prematurely. We set each threshold by examining the training curves of DRATS to give SMT the best possible chance of success, and list all values in Table~\ref{tab:smt_thresholds}.
\item \textbf{Tasks can move between ``solved'' and ``active'' pools.} In the original SMT, once a task is declared solved or unsolvable, it is removed from the active pool and not resampled until the second stage. This is reasonable for off-policy learning, where the replay buffer retains data from inactive tasks and prevents forgetting. In on-policy learning, however, removing a task from the active pool causes the agent to forget it. We therefore allow tasks to move freely between the solved and active pools: if a solved task's return later drops below $M_i$, it is returned to the active pool.
\end{enumerate}

\paragraph{\textsc{Learning Progress} (ALP-GMM).}
ALP-GMM~\citep{portelas2020teacher} was originally proposed for continuous task spaces, where a Gaussian Mixture Model (GMM) is used to estimate absolute learning progress (ALP) across the task parameter space and bias sampling toward high-ALP regions.
In our setting, tasks are discrete, so the GMM is unnecessary: we directly track the absolute learning progress of each task as $z_i = |\widehat{J}_{t,i} - \widehat{J}_{t-1,i}|$ and update the task distribution $\vq$ via mirror ascent with $z_i$ in place of $g_i$.

\paragraph{\textsc{Learning Potential} (PLR).}
PLR~\citep{jiang2021prioritized} originally  constructs a task distribution by applying a rank-based prioritization function to the learning potential scores before normalizing into a distribution.
In our setting, we update the task distribution $\vq$ via mirror ascent with $z_i$ in place of $g_i$ to be consistent with how we update DRATS and \textsc{Learning Progress}, ensuring that differences in performance reflect differences in prioritization strategy rather than differences in how scores are translated into sampling distributions.

\section{Additional Experimental Results}
\label{app:experiments}

\begin{figure}
    \centering
    \begin{subfigure}{\linewidth}
        \includegraphics[width=\linewidth]{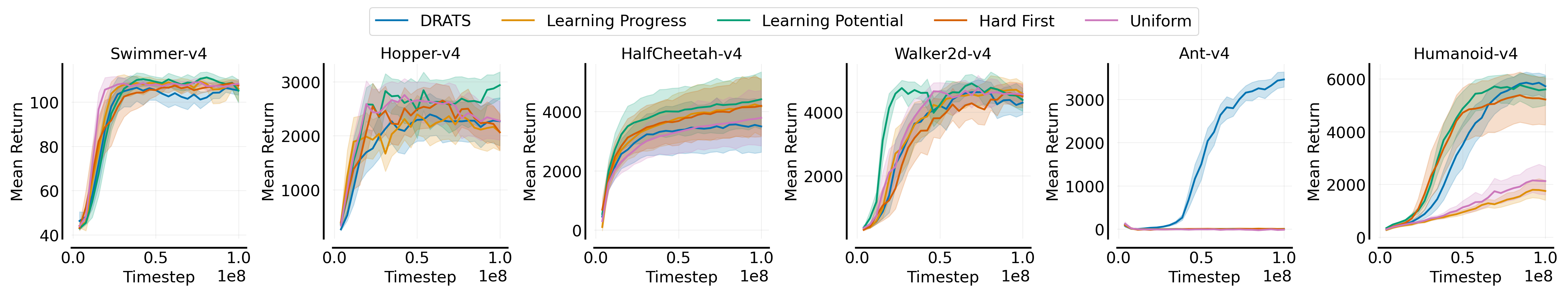}
    \caption{Mean Return in each MuJoCo6 task.}
            \vspace{0.5em}
        \label{fig:mujoco_task_returns}
    \end{subfigure}
    \begin{subfigure}{\linewidth}
        \includegraphics[width=\linewidth]{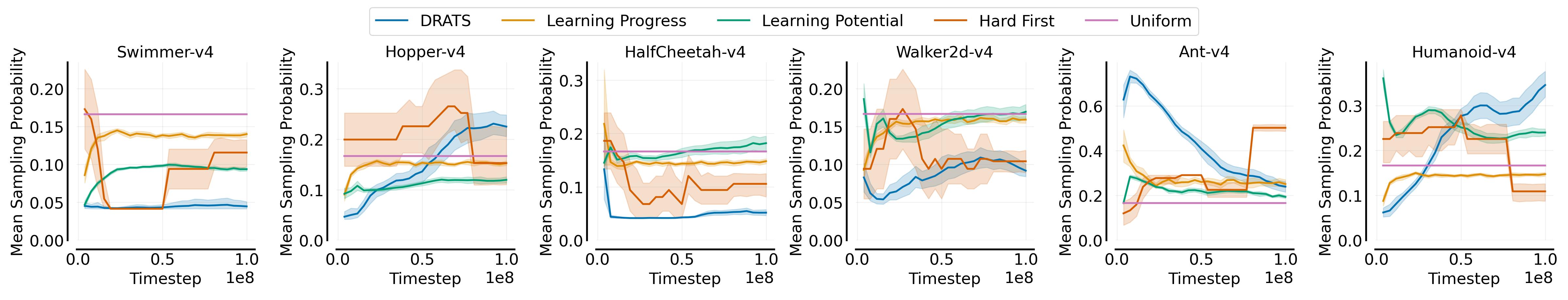}
    \caption{Task sampling probabilities in each MuJoCo6 task.}
            \label{fig:mujoco_task_probs}
    \end{subfigure}
    \caption{Mean return and sampling probability in each MuJoCo6 task. Solid curves denote the mean over 20 seeds and shaded regions denote 95\% bootstrap confidence intervals. 
    }
    \label{fig:mujoco_task_results}
\end{figure}

\begin{figure}
    \centering
    \begin{subfigure}{\linewidth}
        \includegraphics[width=\linewidth]{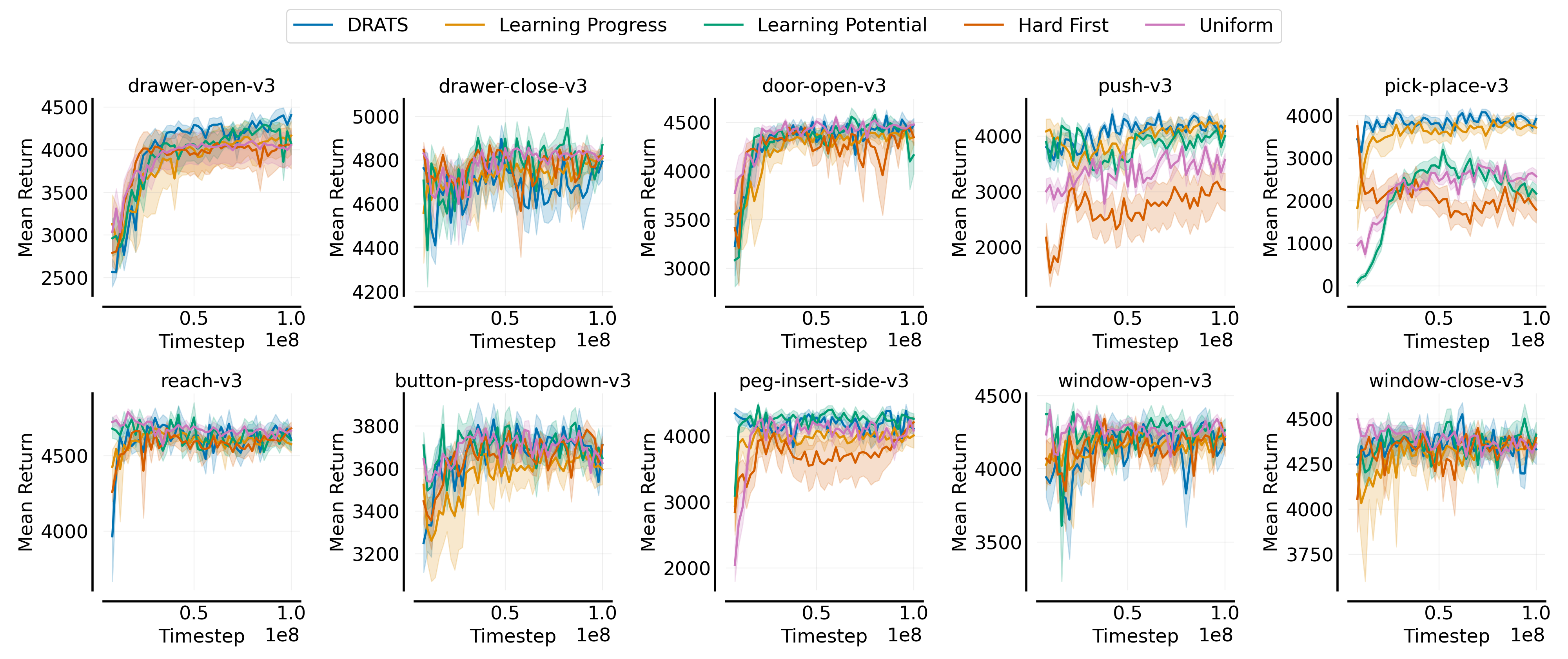}
        \caption{Mean Return in each MT10 task.}
        \label{fig:mt10_return_tasks}
        \vspace{0.5em}
    \end{subfigure}
    \begin{subfigure}{\linewidth}
        \includegraphics[width=\linewidth]{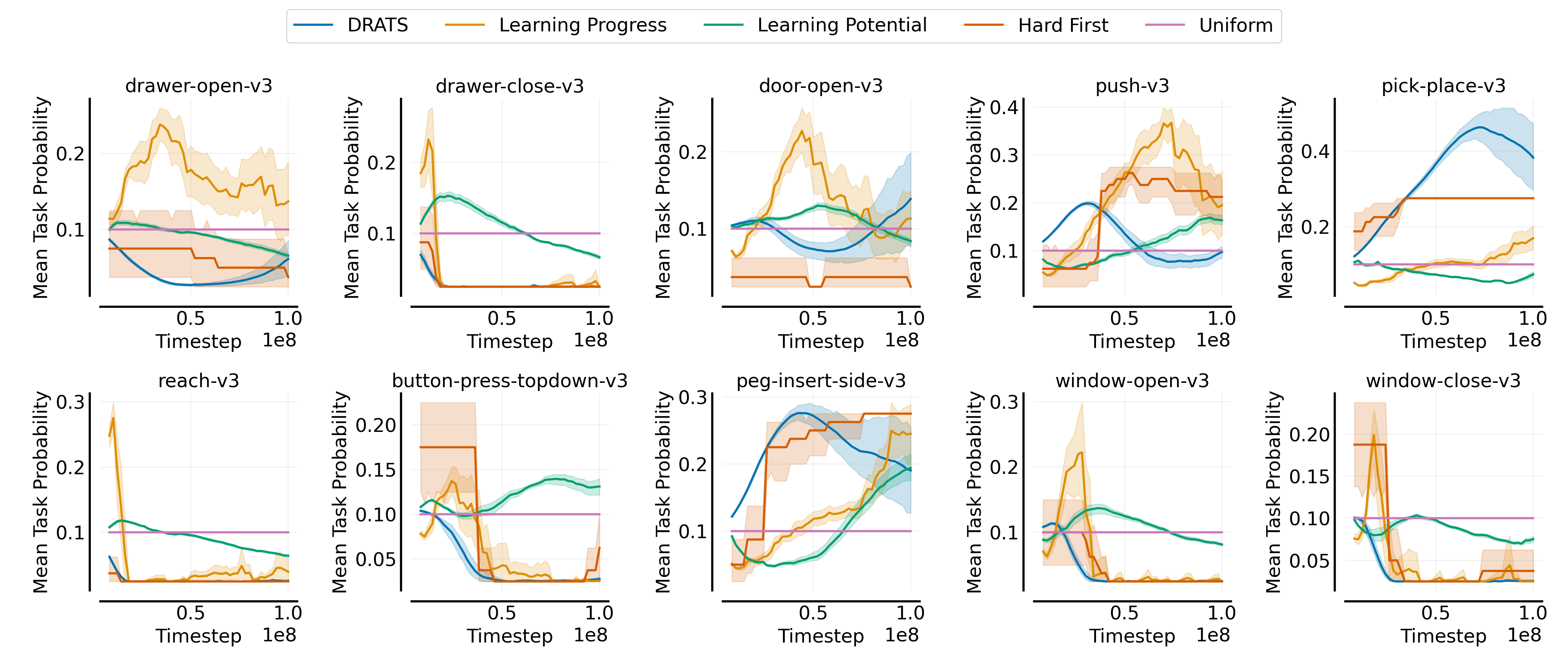}
    \caption{Task sampling probabilities in each MT10 task.}
        \label{fig:mt10_task_probs}
    \end{subfigure}
    \caption{Mean return and sampling probability in each MT10 task. Solid curves denote the mean over 20 seeds and shaded regions denote 95\% bootstrap confidence intervals. 
    }
    \label{fig:mt10_task_results}
\end{figure}

\begin{figure}
    \centering
    \includegraphics[width=\linewidth]{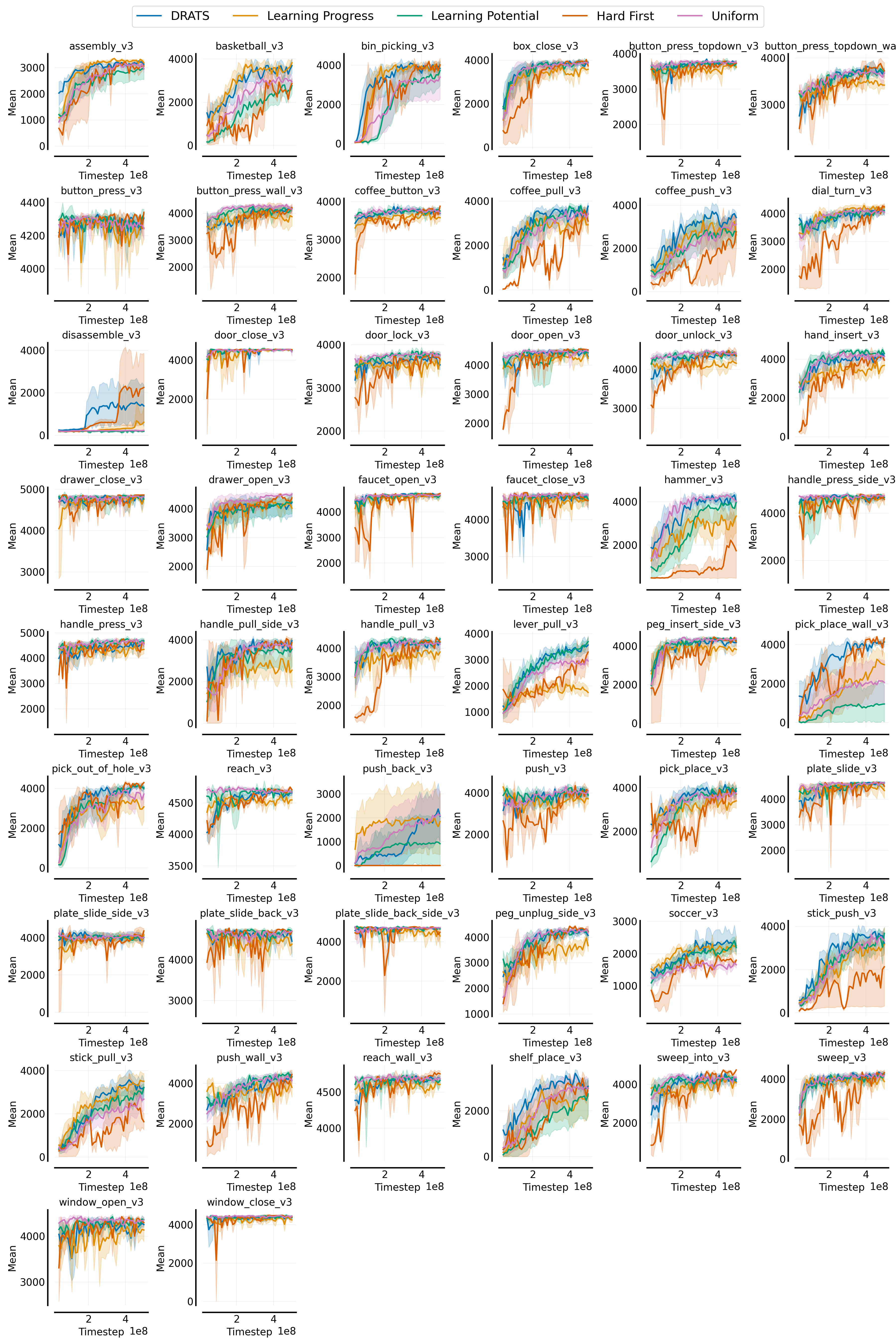}
    \caption{Task returns in MT50 tasks. Solid curves denote the mean over 10 seeds and shaded regions denote 95\% bootstrap confidence intervals.}
    \label{fig:mt50_task_returns}
\end{figure}

\begin{figure}
    \centering
    \includegraphics[width=\linewidth]{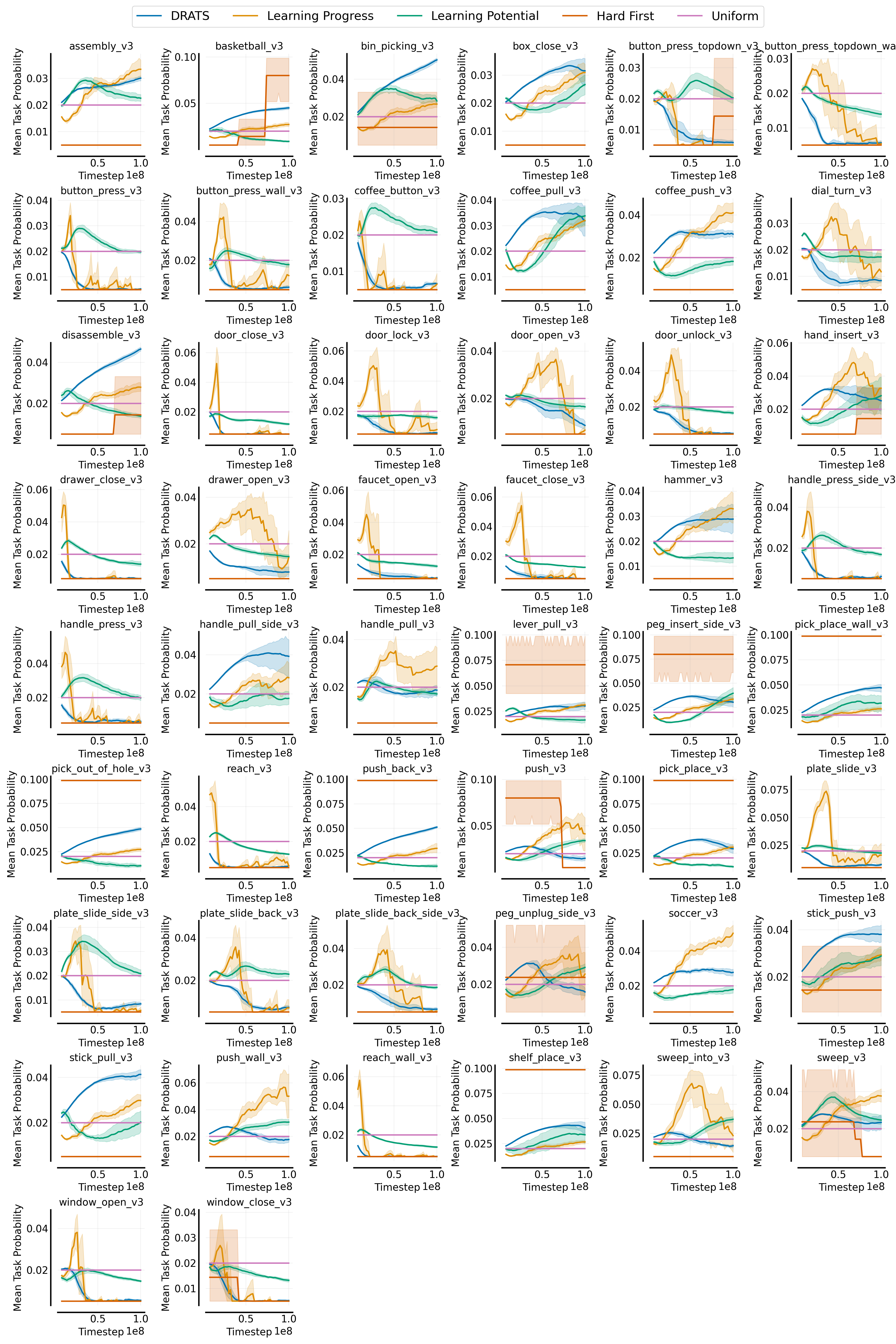}
    \caption{Task sampling probabilities in each MT50 task. Solid curves denote the mean over 10 seeds and shaded regions denote 95\% bootstrap confidence intervals. }
    \label{fig:mt50_task_probs}
\end{figure}

\begin{figure}
    \centering
    \begin{subfigure}{\linewidth}
        \includegraphics[width=\linewidth]{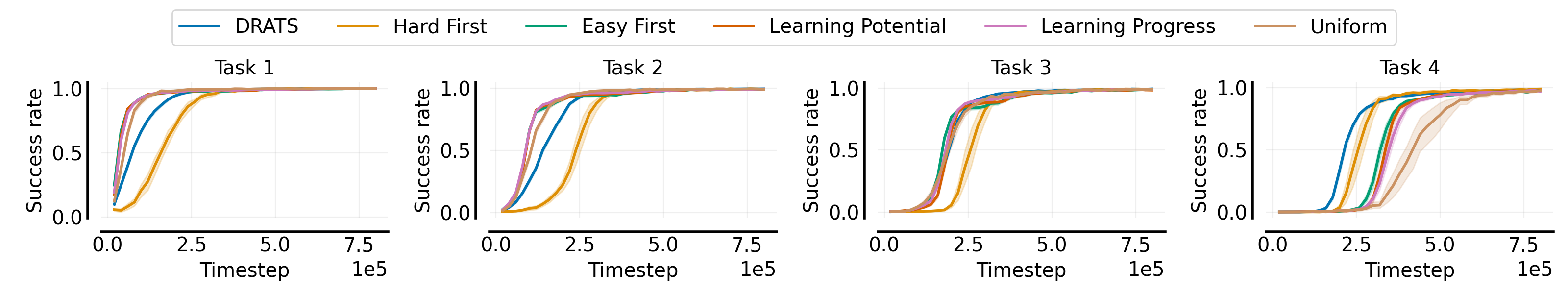}
        \caption{Mean success rate on each task (50 seeds)}
        \label{fig:gridworld_tasks}
    \end{subfigure}
        \begin{subfigure}{\linewidth}
        \includegraphics[width=\linewidth]{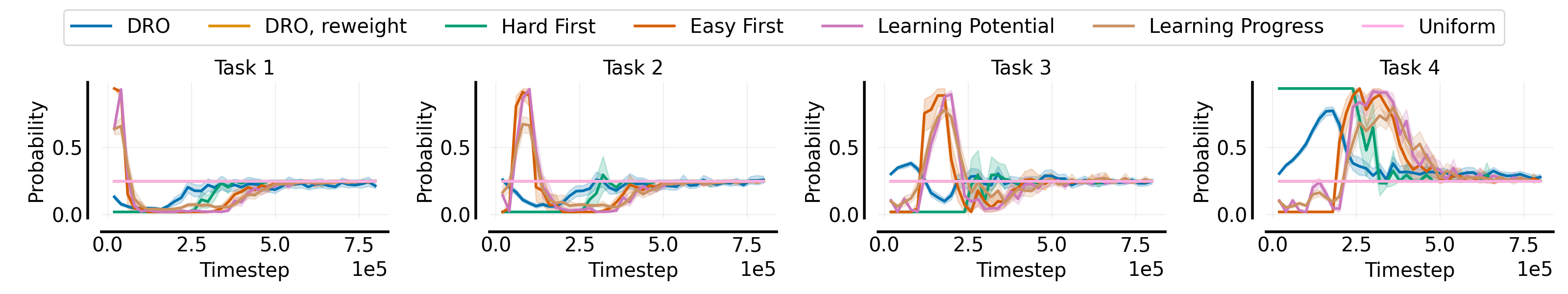}
        \caption{Task sampling probabilities (50 seeds)}
        \label{fig:gridworld_probabilities}
    \end{subfigure}
    \caption{
    \textbf{(a)} . 
    Mean success rates and task sampling probability in each Gridworld task with shared actor/critic networks
    (50 seeds). Shaded regions denote 95\% bootstrap confidence intervals.
    }
\end{figure}

\begin{figure}
    \centering
    \includegraphics[width=0.5\linewidth]{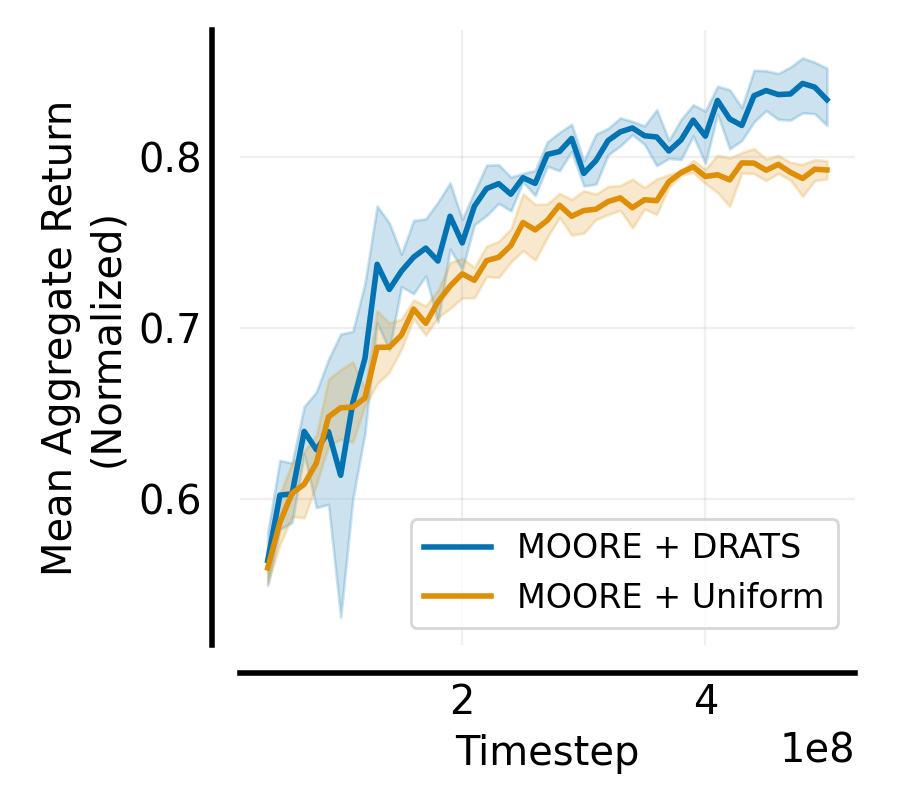}
    \caption{Combing DRATS with MOORE on MT50 (5 seeds). Shaded regions denote 95\% bootstrap confidence intervals.}
    \label{fig:mt50_moore}
\end{figure}

\begin{figure}
    \centering
    \begin{subfigure}{0.45\linewidth}
        \includegraphics[width=\linewidth]{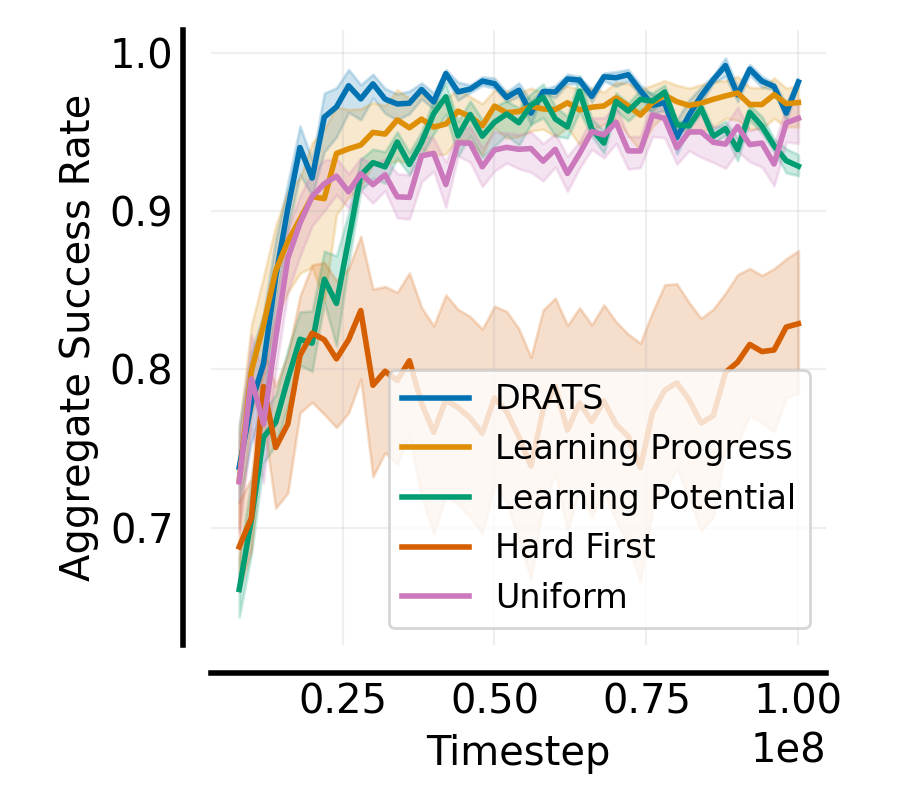}
        \caption{MT10 (20 seeds)}
    \end{subfigure}
    \hfill
        \begin{subfigure}{0.45\linewidth}
        \includegraphics[width=\linewidth]{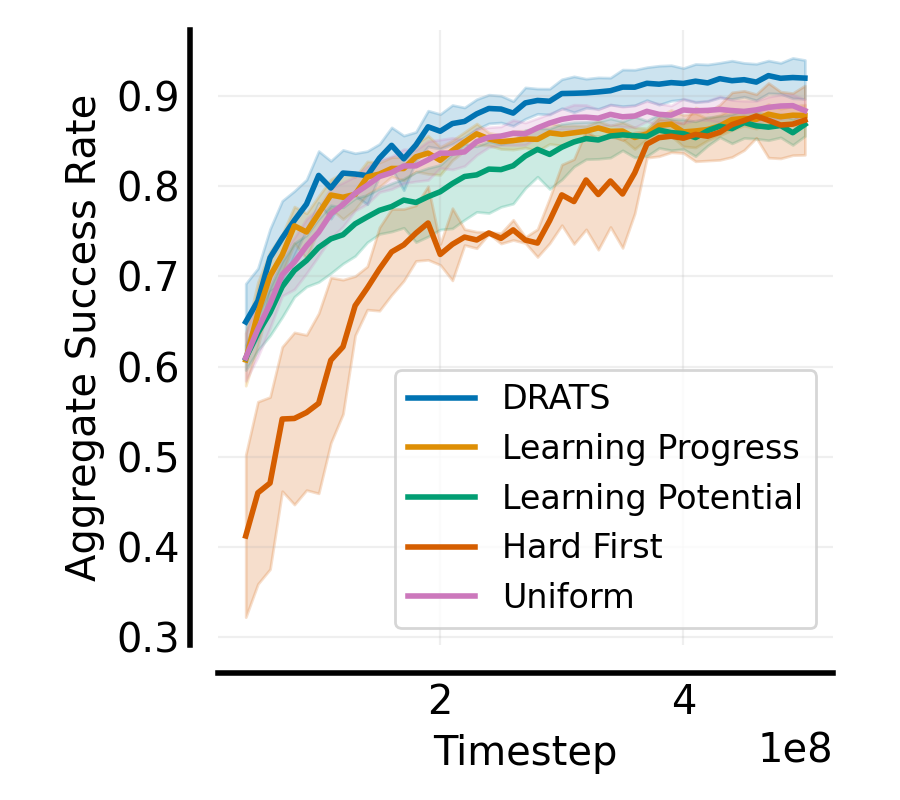}
        \caption{MT50 (10 seeds)}
    \end{subfigure}
    \caption{
    Mean aggregate success rates in MT10 and MT50. Shaded regions denote 95\% bootstrap confidence intervals.
    }
    \label{fig:mt_success_rate}
\end{figure}

\begin{figure}
    \centering
    \includegraphics[width=0.5\linewidth]{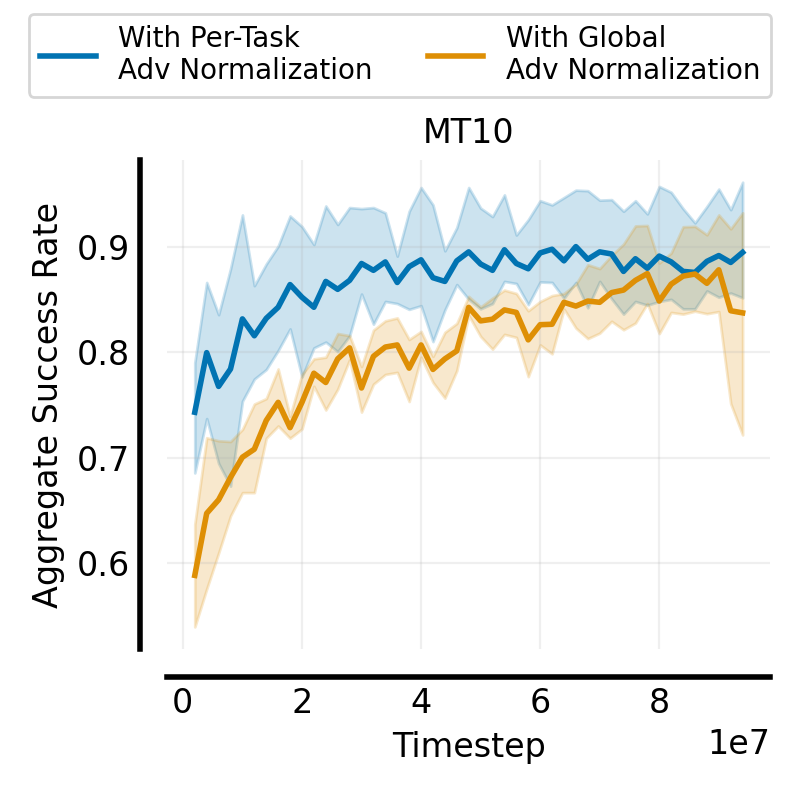}
    \caption{Aggregate success rate of DRATS in MT10 with per-task and global advantage normalization. These training curves use a smaller learning rate of 3e-4 (the default value in the \texttt{metaworld-algorithms} codebase~\citep{mclean2025meta}) whereas all other MT10 experiments use 5e-4.}
    \label{fig:adv_norm}
\end{figure}

\newpage
\section{Ablations}
\label{app:ablations}

\begin{figure}
    \centering
    \begin{subfigure}{0.45\linewidth}
    \includegraphics[width=\linewidth]{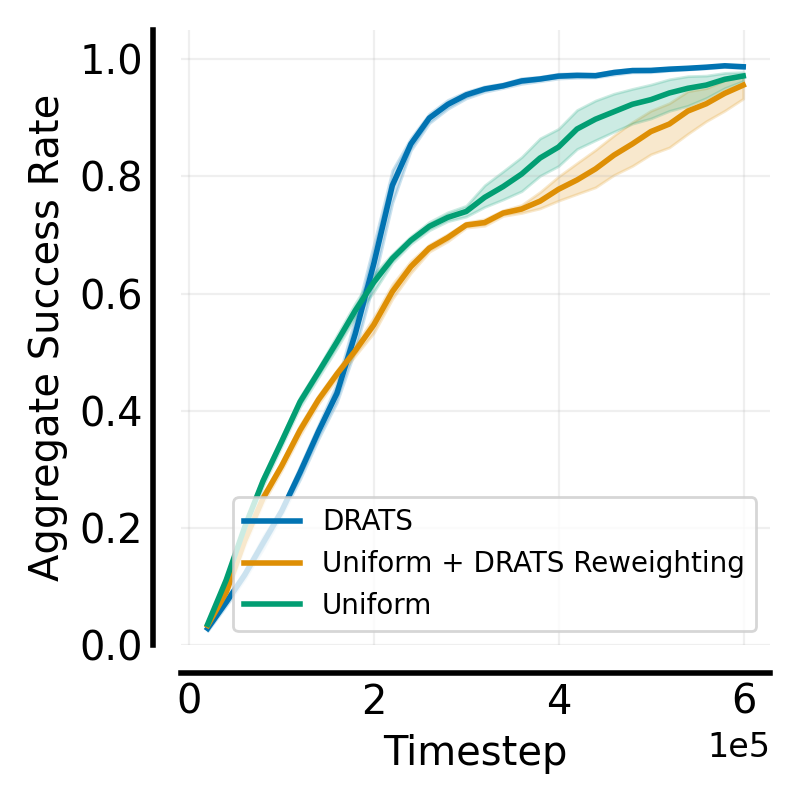}
    \caption{Aggregate success rate}
    \label{fig:gridworld_results_reweight}
    \end{subfigure}
    \hfill
    \begin{subfigure}{0.45\linewidth}
    \includegraphics[width=\linewidth]{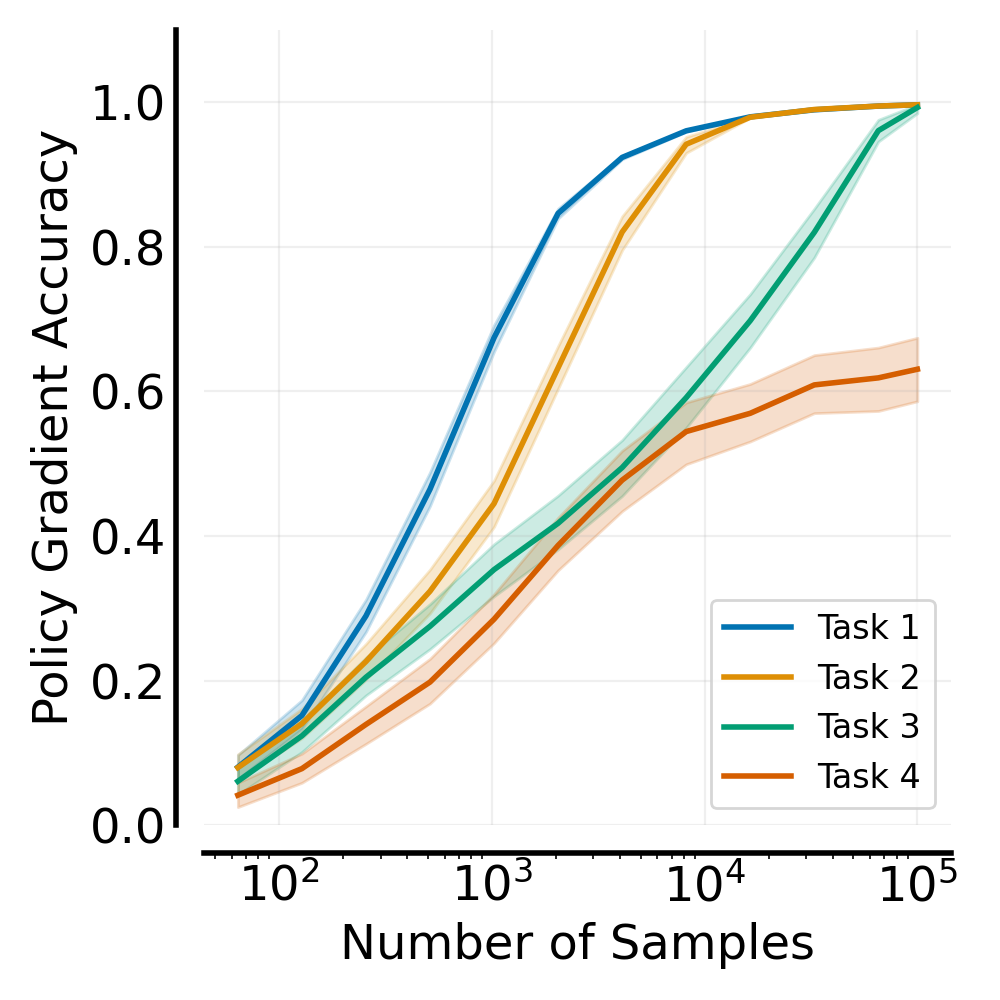}
    \caption{Gradient accuracy under a uniform policy}
    \label{fig:gridworld_accuracy}
    \end{subfigure} \\
    \begin{subfigure}{\linewidth}
    \includegraphics[width=\linewidth]{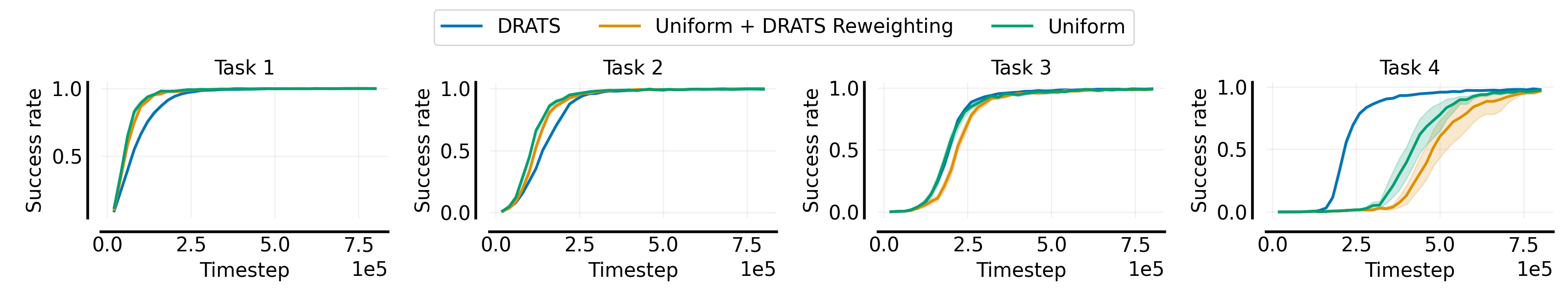}
    \caption{Success rate in each task.}
    \label{fig:gridworld_gradient_accuracy}
    \end{subfigure} 
    \caption{
    Adaptive sampling vs. Objective reweighting on Gridworld (50 seeds). Shaded regions denote 95\% bootstrap confidence belts.
    }
\end{figure}

\subsection{Adaptive Sampling vs. Objective Reweighting}
\label{app:ablations_reweight}

In this section, we compare two ways of optimizing the DRATS objective : adaptive resampling of tasks from $\vq$ (DRATS as described in the main paper), and reweighting the learning objective by $q_i$ while sampling tasks uniformly.

The minimax objective in Eq.~\ref{eq:feas_minimax_sampling} can be optimized in two ways: by resampling tasks from $\vq$ or by sampling tasks uniformly and reweighting each task's loss by $q_i$ . Both are equivalent in expectation, since the expectation in Eq.~\ref{eq:feas_minimax_sampling} can be written as task reweighting: $\min_{\vtheta}\max_{\vq \in \Delta^k} \sum_{i=1}^k q_i g_i(\vtheta)$. Thus, DRATS can be implemented either by drawing tasks from $\vq$ or by reweighting the returns of a uniformly sampled batch by $q_i$ (\textsc{Uniform} + DRATS Reweighting). However, reweighting is known to produce higher-variance gradient estimates than resampling~\citep{an2020resampling}.
As shown in Fig.~\ref{fig:gridworld_results_reweight}, \textsc{Uniform} + DRATS Reweighting is much less data-efficient than DRATS: despite up-weighting hard tasks in the objective, it under-allocates interactions to them and continues sampling easy tasks even after they are solved.
Notably, \textsc{Uniform} + DRATS Reweighting is even less data-efficient than \textsc{Uniform}: by down-weighting easy tasks in the objective, it solves them more slowly than \textsc{Uniform} and therefore benefits less from positive transfer.

Fig.~\ref{fig:gridworld_gradient_accuracy} provides intuition for why resampling yields more data efficient learning than reweighting by plotting the accuracy of policy gradient estimates as a function of sample size, measured as the cosine similarity between a Monte Carlo estimate and the true policy gradient.
Harder tasks require more samples to obtain accurate gradient estimates: achieving a cosine similarity of $0.2$ across all four tasks requires data from each task in a $1:1.5:2:4$ ratio.
Objective reweighting samples tasks uniformly, so gradient estimates are more accurate for easy tasks and less accurate for hard ones.
DRATS with adaptive sampling allocates more data to hard tasks, improving gradient accuracy precisely where it is most needed.

\label{sec:ablations}



\begin{figure}[t]
    \centering
    \begin{subfigure}{0.48\textwidth}
        \centering
        \includegraphics[width=\linewidth]{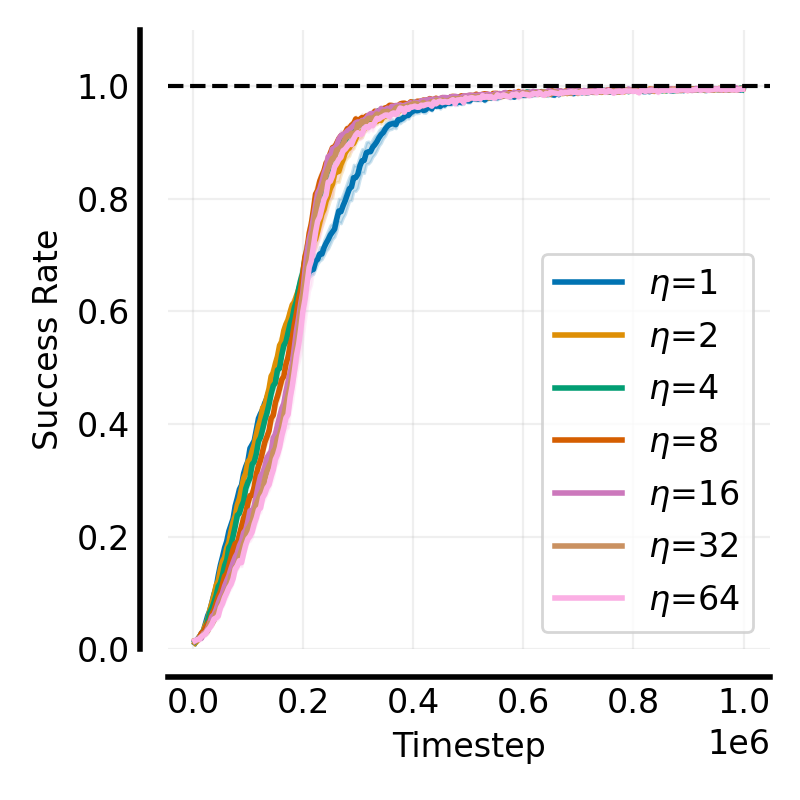}
        \caption{Average Success Rate}
    \label{fig:ablation_drats_eta}
    \end{subfigure}
    \hfill
    \begin{subfigure}{0.48\textwidth}
        \centering
        \includegraphics[width=\linewidth]{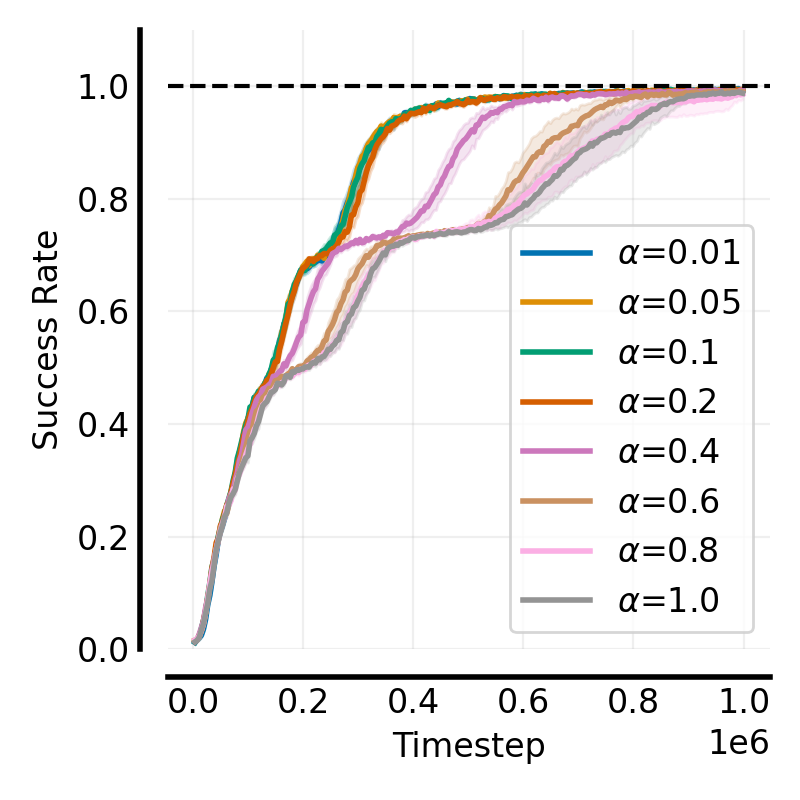}
        \caption{Average Success Rate}
    \label{fig:ablation_drats_alpha}
    \end{subfigure}
    \caption{DRATS ablation on $\eta$ and $\alpha$ (50 seeds). Shaded regions denote 95\% bootstrap confidence intervals.}
\end{figure}

\subsection{DRATS Hyperparameter Ablations}

In this section, we ablate the two core hyperparameters of DRATS: the inverse temperature $\eta$, which controls the sharpness of the task sampling distribution, and the mirror ascent step size $\alpha$, which controls how much the task distribution changes between updates.
%
%
We vary one parameter at a time while keeping the other fixed at its base value, evaluating the following values:
\begin{itemize}
    \item \textbf{Temperature:} $\eta \in \{1, 2, 4, 8, 16, 32, 64\}$
    \item \textbf{Step size:} $\alpha = c\cdot\eta$ for $c \in \{0.01, 0.05, 0.1, 0.2, 0.4, 0.6, 0.8, 1.0\}$
\end{itemize}
A step size of $\alpha = \eta$ (\textit{i.e.}, $c= 1$) corresponds to setting the task distribution to the ``best response'' $\vq^\star$ (Eq.~\ref{eq:softmax_q}), so $c=1$ denotes the maximum valid step size. A step size $\alpha < \eta$ corresponds to take a step toward $\vq^*$.
We note that our hyperparameter sweep on MuJoCo6 in Fig.~\ref{fig:mujoco_hyperparameter_sweep} also ablates $\eta$ for DRATS, \textsc{Learning Progress}, and \textsc{Learning Potential}.

\paragraph{Inverse Temperature ($\eta$).} Fig.~\ref{fig:ablation_drats_eta} shows that DRATS is most data efficient with $\eta = 8$. Smaller values keep the task distribution closer to uniform, reducing the degree to which DRATS prioritizes hard tasks. Larger values concentrate too much probability on the hardest task, which harms data efficiency on easier tasks. Nevertheless, learning curves are nearly identical for $\eta \in [2, 16]$, indicating robustness to this hyperparameter.

\paragraph{Mirror Ascent Step Size ($\alpha$).}
Fig.~\ref{fig:ablation_drats_alpha} shows that data efficiency improves monotonically as $\alpha$ decreases, with optimal performance at $\alpha = 0.01$.
Large step sizes cause the task distribution to shift dramatically between updates, especially when a small number of trajectories are collected between updates. In particular, solved tasks are assigned low sampling probability under DRATS, so if the batch size is small, it may not contain any trajectories from these tasks. In such case, their return estimate defaults to zero and their return gaps thus appears large, causing a large step size to shift substantial probability back onto it despite the task being solved.
Large step sizes cause the task distribution to shift dramatically between updates, which is especially problematic when only a few trajectories are collected per update. In particular, solved tasks are assigned low sampling probability under DRATS, so with a small batch size, the batch may contain no trajectories from these tasks; their return estimates then default to zero, making their return gaps appear large, ultimately causing a large step size to shift substantial probability back onto them despite the tasks being solved.
Smaller step sizes mitigate this issue by taking gradual steps toward $\vq^*$, preventing noisy gap estimates from causing large swings in the sampling distribution.
We note that in our MuJoCo6, MT10, and MT50 experiments, we collect several hundred trajectories between updates, so the DRATS sampling distribution evolves smoothly throughout training even with a larger step size of $\alpha = 0.5 \cdot \eta$ (Figs.~\ref{fig:mujoco_task_probs},~\ref{fig:mt10_task_probs}, and~\ref{fig:mt50_task_probs}).

\section{Hyperparameters}
\label{app:hyperparameters}

\begin{figure}
    \centering
    \begin{subfigure}{0.45\linewidth}
        \includegraphics[width=\linewidth]{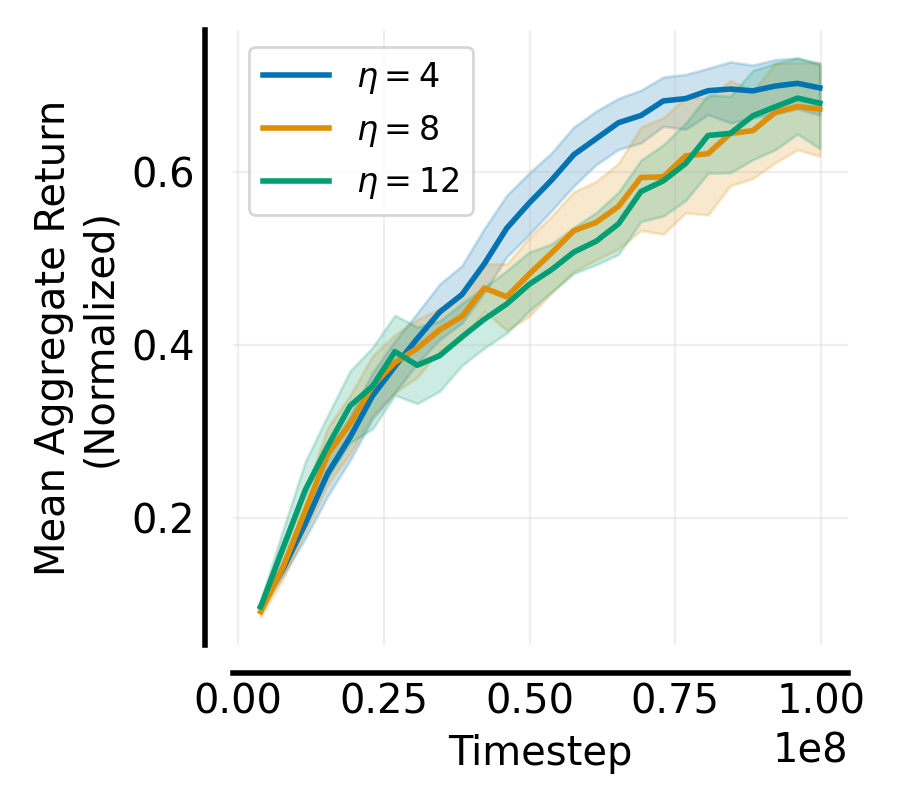}
    \caption{DRATS $\eta$ sweep.}
    \label{fig:dro_eta}
    \end{subfigure}
    \hfill
    \begin{subfigure}{0.45\linewidth}
        \includegraphics[width=\linewidth]{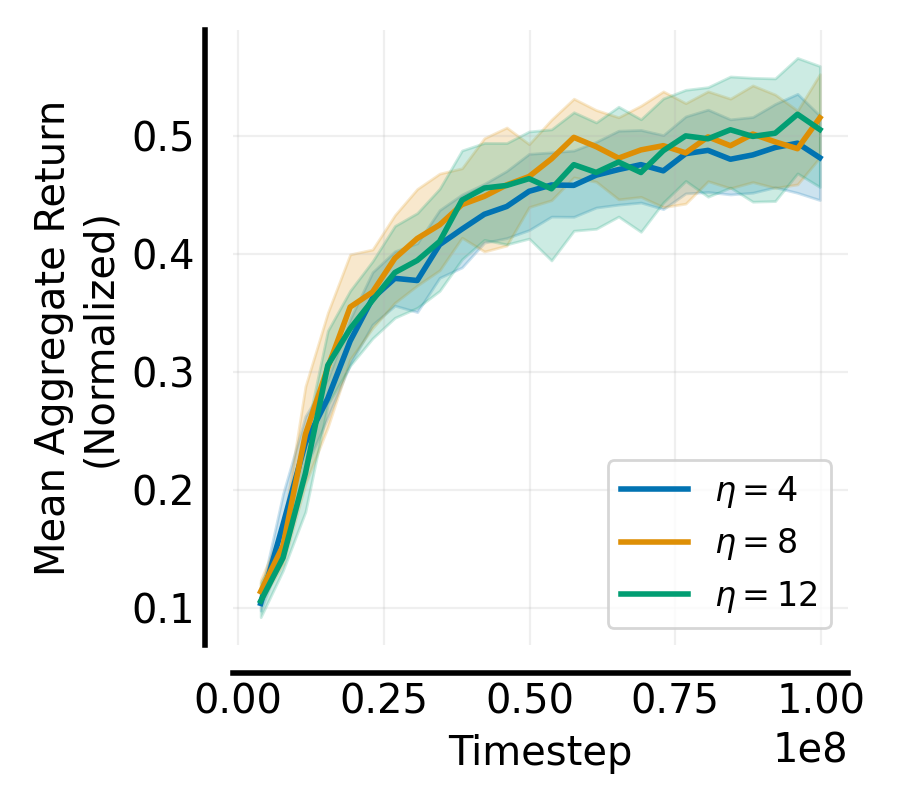}
    \caption{\textsc{Learning Progress} $\eta$ sweep.}
    \label{fig:learning_progress_era}
    \end{subfigure}
    \vspace{0.5em}
    \begin{subfigure}{0.45\linewidth}
        \includegraphics[width=\linewidth]{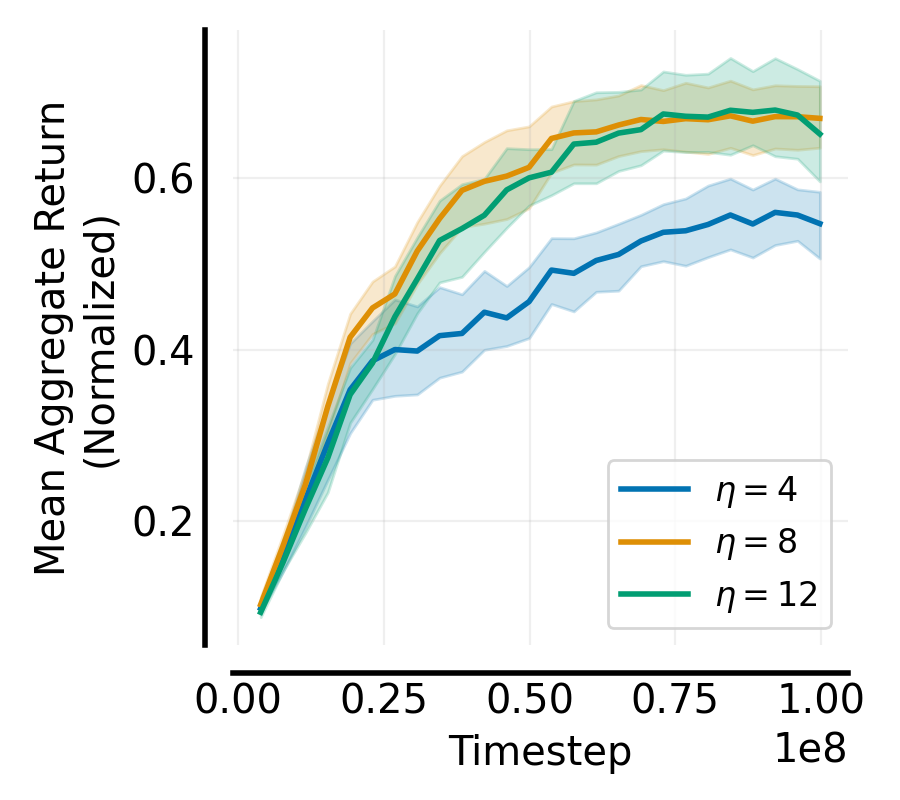}
    \caption{\textsc{Learning Potential} $\eta$ sweep.}
    \label{fig:learning_potential_eta}
    \end{subfigure}
    \hfill
    \begin{subfigure}{0.45\linewidth}
        \includegraphics[width=\linewidth]{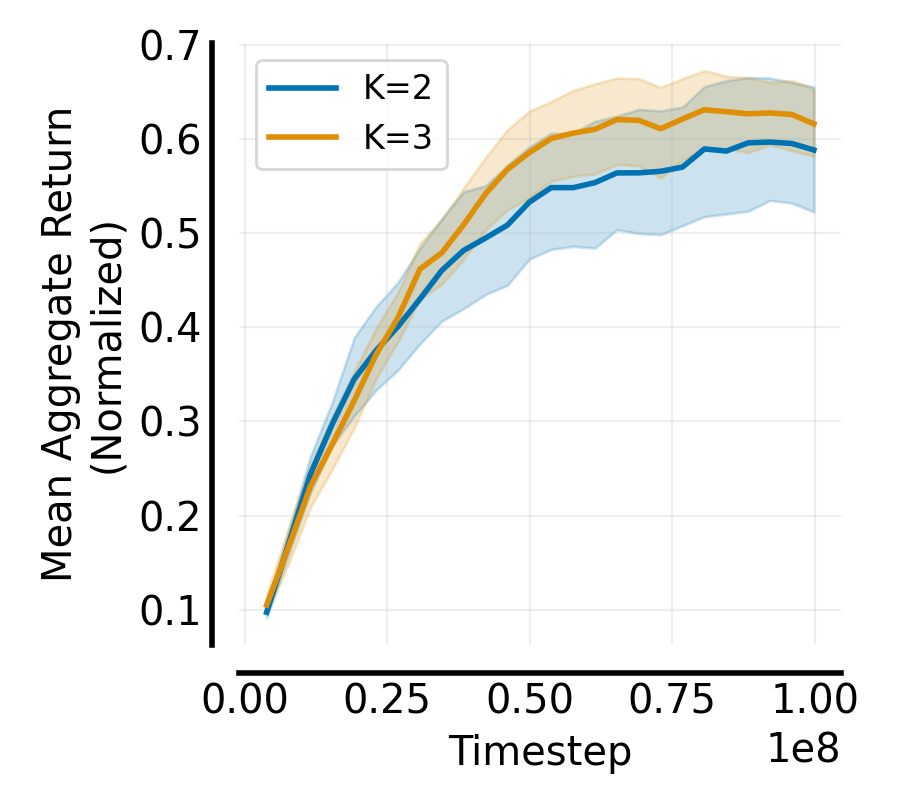}
    \caption{\textsc{Hard First} $K$ sweep.}
    \label{fig:hard_first_k}
    \end{subfigure}
    \caption{Hyperparameter sweeps in MuJoCo6. Shaded regions denote 95\% bootstrap confidence intervals.
    }
    \label{fig:mujoco_hyperparameter_sweep}
\end{figure}

In this section, we describe how we set hyperparameters for each baseline. We tune MuJoCo6 hyperparameters via parameter sweep; since MT10 and MT50 experiments are significantly more expensive, we use a single hyperparameter setting for each, following the original papers where possible. We fix the mirror ascent step size to $\alpha = 0.5\cdot\eta$ for all methods. All hyperparameters are summarized in Tables~\ref{tab:mtppo_hyperparams} and~\ref{tab:mujoco_hyperparams}.

\paragraph{DRATS.}
For MuJoCo6, we sweep $\eta \in \{4, 8, 12\}$ and report results for $\eta = 4$. For MT10 and MT50, we use $\eta = 8$ and $\eta = 3$, respectively.

\paragraph{\textsc{Hard First} (SMT,~\citet{cho2024hard}).}
For MuJoCo6, we sweep $K \in \{2, 3\}$ with $B_1 = 80\%$ and find $K=3$ is most data efficient. For MT10 and MT50, we follow \citet{cho2024hard} and use $K=3$ and $K=8$ with $B_1 = 85\%$. Thresholds are listed in Table~\ref{tab:smt_thresholds}.

\paragraph{\textsc{Learning Progress} (ALP-GMM,~\citet{portelas2020teacher}).}
For MuJoCo6, we sweep $\eta \in \{4, 8, 12\}$ and report results for $\eta = 12$. For MT10 and MT50, we use $\eta = 12$ and $\eta = 3$, respectively. 

\paragraph{\textsc{Learning Potential} (PLR,~\citet{jiang2021prioritized}).}
For MuJoCo6, we sweep $\eta \in \{4, 8, 12\}$ and report results for $\eta = 8$. For MT10 and MT50, we use $\eta = 5$. We choose $\eta$ larger than DRATS's for MT10/MT50 because our MuJoCo6 sweep showed \textsc{Learning Potential} performed best with a larger value than DRATS.


\begin{table}[t]
\centering
\begin{tabular}{ll}
\toprule
\textbf{Description} & \textbf{Value} \\
\midrule

Episode length & 500 \\
Batch size & 100,000 (200 episodes) / 500,000 (1000 episodes)  \\
Number of epochs & 16 \\
Number of mini-batches per epoch & 32 \\
Discount factor & 0.99 \\
Clip ratio & 0.2 \\
Policy entropy coefficient & 0.01 / 0.03 \\
Optimizer learning rate & 5e-4 \\
Advantage estimate tau & 0.95 \\
Value Normalization by task & False \\
Input Normalization by task & False \\
Per-task Advantage Normalization & True \\
Separate critic and policy networks & True \\
Network Architecture & Multi-head MLP with layers (400, 400, 400) \\
\midrule
\multicolumn{2}{l}{\textbf{DRATS-Specific Hyperparameters}} \\
\midrule
Inverse temperature $\eta$ & 8 / 3\\
Mirror ascent step size $\alpha$ & $0.5\cdot \eta$\\
\midrule
\multicolumn{2}{l}{\textbf{Learning Progress-Specific Hyperparameters}} \\
\midrule
Inverse temperature $\eta$ & 8 / 3\\
Mirror ascent step size $\alpha$ & $0.5\cdot \eta$\\
\midrule
\multicolumn{2}{l}{\textbf{Learning Potential-Specific Hyperparameters}} \\
\midrule
Inverse temperature $\eta$ & 8 / 3 \\
Mirror ascent step size $\alpha$ & $0.5\cdot \eta$\\
\midrule
\multicolumn{2}{l}{\textbf{Hard First-Specific Hyperparameters}} \\
\midrule
Active task set size $K$ & 3 / 8 \\
Stage 1 budget $B_1$ & 85\% of total training budget \\
Thresholds $m_i, M_i$ & See Table~\ref{tab:smt_thresholds} \\ 
\midrule
\multicolumn{2}{l}{\textbf{MOORE-Specific Hyperparameters}} \\
\midrule
MoE experts & 2 \\
MoE layers & 2 \\
MoE hidden dim & 256 \\
Activation before/after task encoding weighting & [Linear, Linear] \\
Task encoder hidden sizes & [128] \\
Task encoder bias & False \\
\midrule
\multicolumn{2}{l}{\textbf{Soft-Modularization-Specific Hyperparameters}} \\
\midrule
MoE experts & 2 \\
MoE layers & 2 \\
State encoder hidden sizes & 256 \\
Task encoder hidden sizes & 256 \\

\bottomrule
\end{tabular}
\caption{Hyperparameters used in MT10 / MT50.}
\label{tab:mtppo_hyperparams}
\end{table}

\begin{table}[t]
\centering
\begin{tabular}{ll}
\toprule
\textbf{Description} & \textbf{Value} \\
\midrule

Episode length & 1000 \\
Batch size & 48000 $\times$ 4 parallel environments  \\
Number of epochs & 16 \\
Number of mini-batches per epoch & 32 \\
Discount factor & 0.99 \\
Clip ratio & 0.2 \\
Policy entropy coefficient & 0.01  \\
Optimizer learning rate & 1e-4 \\
Advantage estimate tau & 0.95 \\
Value Normalization by task & True \\
Input Normalization by task & True \\
Per-task Advantage Normalization & True \\
Separate critic and policy networks & True \\
Network Architecture & Multi-head MLP with layers (256, 256, 256) \\
\midrule
\multicolumn{2}{l}{\textbf{DRATS-Specific Hyperparameters}} \\
\midrule
Inverse temperature $\eta$ & 4\\
Mirror ascent step size $\alpha$ & $0.5\cdot \eta$\\
\midrule
\multicolumn{2}{l}{\textbf{Learning Progress-Specific Hyperparameters}} \\
\midrule
Inverse temperature $\eta$ & 8\\
Mirror ascent step size $\alpha$ & $0.5\cdot \eta$\\
\midrule
\multicolumn{2}{l}{\textbf{Learning Potential-Specific Hyperparameters}} \\
\midrule
Inverse temperature $\eta$ & 8 \\
Mirror ascent step size $\alpha$ & $0.5\cdot \eta$\\
\midrule
\multicolumn{2}{l}{\textbf{Hard First-Specific Hyperparameters}} \\
\midrule
Active task set size $K$ & 3  \\
Stage 1 budget $B_1$ & 80\% of total training budget \\
Thresholds $m_i, M_i$ & See Table~\ref{tab:smt_thresholds} \\ 
\bottomrule
\end{tabular}
\caption{RL hyperparameters used in MuJoCo6.}
\label{tab:mujoco_hyperparams}
\end{table}

\begin{table}[t]
\centering
\begin{tabular}{lcc}
\toprule
\textbf{Task} & $m_i$ & $M_i$ \\
\midrule
Swimmer-v4      & 60   & 90   \\
Hopper-v4       & 1500 & 2500 \\
HalfCheetah-v4  & 1500 & 3000 \\
Walker2d-v4     & 1500 & 4000 \\
Ant-v4          & 1000 & 3000 \\
Humanoid-v4     & 2000 & 5000 \\
\midrule
assembly-v3 & 2000 & 3000 \\
basketball-v3 & 2000 & 2000 \\
bin-picking-v3 & 2000 & 3000 \\
box-close-v3 & 2000 & 3000 \\
button-press-topdown-v3 & 2000 & 3000 \\
button-press-topdown-wall-v3 & 2000 & 3000 \\
button-press-v3 & 2000 & 3500 \\
button-press-wall-v3 & 2000 & 3500 \\
coffee-button-v3 & 2000 & 3000 \\
coffee-pull-v3 & 2000 & 2000 \\
coffee-push-v3 & 2000 & 2000 \\
dial-turn-v3 & 2000 & 3500 \\
disassemble-v3 & 2000 & 3000 \\
door-close-v3 & 2000 & 4000 \\
door-lock-v3 & 2000 & 3000 \\
door-open-v3 & 2000 & 4000 \\
door-unlock-v3 & 2000 & 3000 \\
hand-insert-v3 & 2000 & 3500 \\
drawer-close-v3 & 2000 & 4000 \\
drawer-open-v3 & 2000 & 4000 \\
faucet-open-v3 & 2000 & 4000 \\
faucet-close-v3 & 2000 & 4000 \\
hammer-v3 & 2000 & 3000 \\
handle-press-side-v3 & 2000 & 4000 \\
handle-press-v3 & 2000 & 4000 \\
handle-pull-side-v3 & 2000 & 3000 \\
handle-pull-v3 & 2000 & 3500 \\
lever-pull-v3 & 2000 & 2500 \\
pick-place-wall-v3 & 2000 & 3000 \\
pick-out-of-hole-v3 & 2000 & 3000 \\
pick-place-v3 & 2000 & 3000 \\
plate-slide-v3 & 2000 & 4000 \\
plate-slide-side-v3 & 2000 & 3500 \\
plate-slide-back-v3 & 2000 & 4000 \\
plate-slide-back-side-v3 & 2000 & 4000 \\
peg-insert-side-v3 & 2000 & 3500 \\
peg-unplug-side-v3 & 2000 & 3500 \\
soccer-v3 & 2000 & 2000 \\
stick-push-v3 & 2000 & 2500 \\
stick-pull-v3 & 2000 & 2000 \\
push-v3 & 2000 & 3000 \\
push-wall-v3 & 2000 & 3000 \\
push-back-v3 & 2000 & 3000 \\
reach-v3 & 2000 & 4000 \\
reach-wall-v3 & 2000 & 4000 \\
shelf-place-v3 & 2000 & 3000 \\
sweep-into-v3 & 2000 & 3500 \\
sweep-v3 & 2000 & 3000 \\
window-open-v3 & 2000 & 4000 \\
window-close-v3 & 2000 & 4000 \\
\bottomrule
\end{tabular}
\caption{\textsc{Hard First} (SMT) unsolvable ($m_i$) and solved ($M_i$) thresholds for all tasks.}
\label{tab:smt_thresholds}
\end{table}

\newpage
\clearpage
\section{Computational Resources}
\label{app:compute}
A full MT10 or MT50 training run takes 20-30 hours on a single A100 GPUs with 80GB of memory. We have access to a computing cluster with many GPUs and train each agent on separate GPU and use parallel environments across 10 CPUs. MT10 runs require 12GB memory and 15GB disk; MT50 runs require 85GB memory and 15GB disk.
We run MuJoCo6 experiments on 4 CPUs only, and each training run takes approximately 40 hours.

\end{document}